\documentclass[accepted]{uai2026} 
                        
\usepackage[american]{babel}

\usepackage{natbib} 
\usepackage{mathtools} 
\usepackage{booktabs} 
\usepackage{tikz} 

\usepackage{algorithm}
\usepackage{algorithmic}

\usepackage{tikz}
\usetikzlibrary{positioning}
\usetikzlibrary{shapes, arrows, arrows.meta, positioning}
\usepackage{caption}
\usepackage{subcaption}

\usepackage{amsmath}
\usepackage{amssymb}
\usepackage{mathtools}
\usepackage{amsthm}
\usepackage{bbm}

\usepackage[capitalize,noabbrev]{cleveref}
\usepackage{thmtools}

\newtheorem{theorem}{Theorem}[section]

\newtheorem{lemma}[theorem]{Lemma}
\newtheorem{corollary}[theorem]{Corollary}
\theoremstyle{definition}
\newtheorem{definition}[theorem]{Definition}

\theoremstyle{remark}
\newtheorem{remark}[theorem]{Remark}

\usepackage{wrapfig}
\usepackage{bbm}
\DeclareMathOperator*{\argmax}{arg\,max}

\newcommand{\V}[0]{\mathcal{V}}
\newcommand{\ocal}[0]{\mathcal{O}}
\newcommand{\E}[0]{\mathbb{E}}      
\newcommand{\ev}[0]{\mathcal{E}}    
\newcommand{\pr}[0]{\mathbb{P}}
\newcommand{\G}[0]{\mathcal{G}}
\newcommand{\mui}[0]{\mu(A_{N(i)})}
\newcommand{\npj}[0]{n_{tP_j}(A)}
\newcommand{\npjT}[0]{n_{TP_j}(A)}
\newcommand{\npjs}[0]{n_{tP_j}(A^*)}
\newcommand{\npjt}[0]{n_{tP_j}(A_t)}
\newcommand{\rtgpi}[1]{R_T(\pi, #1)}

\usepackage{tikz}
\usetikzlibrary{positioning}

\title{Graph-Dependent Regret Bounds in Multi-Armed Bandits with Interference}

\author[1]{\href{mailto:fateme.jamshidi@epfl.ch}{Fateme~Jamshidi \thanks{Equal contribution.}}}
\author[1]{\href{mailto:mohammad.shahverdikondori@epfl.ch}{Mohammad~Shahverdikondori \protect\footnotemark[1]}}
\author[1]{\href{mailto:negar.kiyavash@epfl.ch}{Negar~Kiyavash}}
\affil[1]{%
    College of Management of Technology\\
    EPFL\\
    Lausanne, Switzerland
}

\begin{document}
\maketitle

\begin{abstract}
We study multi-armed bandits under network interference, where each unit’s reward depends on its own treatment and those of its neighbors in a given graph. This induces an exponentially large action space, making standard approaches computationally impractical. 
We propose a novel algorithm that uses the local graph structure to minimize regret.
We derive a graph-dependent upper bound on cumulative regret that improves over prior work.
Additionally, we provide the first lower bounds for bandits with arbitrary network interference, where each bound involves a distinct structural property of the graph.
These bounds show that for both dense and sparse graphs, our algorithm is nearly optimal, with matching upper and lower bounds up to logarithmic factors.
When the interference graph is unknown, a variant of our algorithm is Pareto optimal: no algorithm can uniformly outperform it across all instances.
We complement our theoretical results with numerical experiments, showing that our approach outperforms the baseline methods.
\end{abstract}

\section{Introduction}
Multi-armed bandits (MABs) have become a fundamental tool for online decision-making across a variety of applications \citep{mab1-bouneffouf2020survey, mab2-tewari2017ads, lattimore2020bandit, pure-bubeck2009pure}.
Although classic MABs perform well in certain settings, they are prone to systematic errors when \emph{interference} is present.
Interference arises when a unit’s outcome depends not only on its own treatment but also on the treatment of others.
For instance, the vaccination status of an individual can affect whether others fall ill, or a personal advertisement may influence friends of the targeted user.
Such interference poses challenges to traditional experimental design and sequential learning frameworks, which often assume independent unit responses; formally, the \textit{Stable Unit Treatment Value Assumption} (SUTVA) \citep{rubin1980randomization}.
Violations of SUTVA with cross-unit dependencies complicates the applicability of traditional methods, leading to suboptimal decisions. 

In recent years, certain methods have been developed that account for interference in settings where only the quality of the final output is of interest \citep[e.g.,][]{ugander2013graph, athey2018exact}.
The cumulative performance of the experimentation, which is particularly relevant in most online settings, has received relatively less attention, in part because it is more difficult to analyze. 
We aim to advance this line of research by studying the MAB problem in the presence of arbitrary network interference.
Specifically, we consider $N$ units representing entities such as users on an online platform or a medical trial. 
There are $k$ treatment arms (or arms) and a time horizon of $T$ rounds.
At each round, the learner assigns one arm to each unit and observes the resulting reward, with the goal of maximizing the total accumulated reward over all rounds. 
Unlike the traditional MAB framework, where rewards are assumed to be independent, interference introduces a dependency structure, with the reward of each unit determined not only by its own arm but also by the arms assigned to others.
This dependency can significantly increase the complexity of exploration, as the action space may include up to $k^N$ different actions, growing exponentially in the number of units.
  
Classical MAB algorithms, such as the Upper Confidence Bound (UCB) method \citep{auer2002finite}, yield a regret scaling of $\tilde{\mathcal{O}}(\sqrt{\frac{T}{N}k^N})$\footnote{Assuming each unit’s reward is 1-sub-Gaussian, the average reward over $N$ units follows a $(1/\sqrt{N})$-sub-Gaussian distribution.}, which becomes prohibitive as $N$ grows.
Furthermore, without imposing any assumptions on the interference structure, the regret is guaranteed to scale as $\Omega(\sqrt{\frac{T}{N}k^N})$, as shown by lower bounds in the MAB literature \citep{lattimore2020bandit}.

To address this challenge, we consider a setting where the reward of any given unit is influenced by its own arm and the arms assigned to its neighboring units, as defined by a graph that represents the interferences among the units.
This graph is referred to as the \textit{interference graph}.

\noindent\textbf{Contributions}
Our main contributions are as follows.
\begin{itemize}
    \item We introduce the Partitioned Upper Confidence Bound with Interference (PUCB-I) algorithm for the MAB problem in the presence of network interference.
    Using the local structure of the interference graph, we establish a graph-dependent upper bound on cumulative regret. 
    The algorithm outperforms state-of-the-art methods in all of the settings.
    \item We derive the first lower bounds on the regret for bandits with arbitrary network interference.
    We provide two distinct lower bounds that quantify the worst-case regret of any algorithm based on the topological properties of the underlying interference graph.
    Specifically, these properties pertain to i) the symmetries in neighborhoods of the nodes (representing the units) and ii) the structure of doubly-independent sets of the interference graph.
    We demonstrate that the upper bound of our algorithm matches these lower bounds up to logarithmic factors for both dense and sparse graphs, thereby establishing its near-optimality.
    Furthermore, in general, the gap between the upper and lower bounds is at most $\sqrt{N}$. 
    \item We extend our analysis to the setting where the interference graph is unknown. 
    We introduce a variant of our algorithm, PUCB-UI, which assumes the true interference graph is complete and runs PUCB-I accordingly.
    We show that PUCB-UI is in the Pareto optimal set of algorithms, meaning that no algorithm can uniformly outperform it across all instances.
    This result establishes a fundamental limit on learning under an unknown interference structure.
    \item  Through simulations, we show that our algorithm outperforms baseline methods.

\end{itemize}

\noindent\textbf{Paper Outline}
We begin with related work, followed by background and key definitions in Section~\ref{sec: background}.
Section~\ref{sec: upper} presents our algorithm and its regret upper bound, and Section~\ref{sec: lower} provides matching lower bounds. 
Section ~\ref{sec: unkown} addresses the setting with an unknown interference graph.
Simulations appear in Section ~\ref{sec: exp} and all the proofs are included in Appendix~\ref{apx:proofs}.

\subsection{Related Work}
Cross-unit interference has been studied in statistics \citep{hudgens2008toward, eckles2017design, basse2018analyzing, li2022random, leung2023network}, computer science \citep{ugander2013graph, yuan2021causal, ugander2023randomized} and medical research \citep{tchetgen2012causal}.
To address the challenge of interfering units, researchers have proposed tailor-made methodologies in structured interference models of their specific settings, such as intragroup interference \citep{rosenbaum2007interference, hudgens2008toward}, network neighborhoods \citep{ugander2013graph, bhattacharya2020causal, yu2022estimating, gao2023causal}, bipartite graphs \citep{pouget2019variance, bajari2021multiple, bajari2023experimental} and general graph models through exposure mappings
\citep{aronow2012general, aronow2017estimating, elena1-adhikari2025learning}.

Despite this extensive literature, all the aforementioned works have focused merely on strategies that maximize the final reward.
The harder problem of understanding cumulative performance over time remains relatively unexplored. 
The MAB framework is naturally well-suited to address this gap by balancing exploration and exploitation. 
Studies in structured bandits have shown that incorporating problem structure, such as dependencies among actions or feedback,
can lead to more efficient learning \citep{chen2014combinatorial, lattimore2016causal, alon2017nonstochastic, jamshidi2024confounded, shahverdikondori2025optimal}.
In our case, the structure arises from interference between units.
A recent work by \citet{zhang2024online} studied online experimental design under network interference, focusing on a trade-off between reward and estimation.
In contrast, our focus is on minimization of regret in sequential learning under arbitrary interference.
\cite{elena2-faruk2025leveraging} studied a contextual MAB setting with interference in a social network. In their framework, at each round, a single unit is assigned a treatment, which may activate its neighboring units randomly, with the objective of maximizing the total number of activated nodes over time. The linear contextual bandit setting with network interference is also studied recently by \cite{contextual-xu2024linear}.
Our setting can be formulated as a specific instance of the combinatorial bandit framework \citep{cesa2012combinatorial, cucb-chen2013combinatorial}.
However, we show that no algorithm achieving optimal regret in the general combinatorial setting can be optimal in ours. We show this by deriving an upper bound that improves over the matching upper and lower bounds known for combinatorial bandits (see Remark \ref{remark: comb}). 
These results underscore the theoretical significance of our setting as a structured subclass that merits separate analysis.
Another line of work is multiple-play bandits \citep{anantharam1987asymptotically, chen2013combinatorial, lagree2016multiple, jia2023short} that considers settings where the learner selects multiple arms simultaneously and observes feedback for each arm. 
However, they assume that the rewards are independent across arms, which is not the case in the presence of interference.

Recently, \citet{jia2024multi} studied a batched adversarial bandit framework where $N$ units lie on the $\sqrt{N} \times \sqrt{N}$ unit grid.
By limiting the action space to those with an identical arm for all units, they achieved a regret bound that does not depend exponentially on $N$. 
However, such an approach is quite limiting in practice, as the optimal action may be one that assigns heterogeneous arms to units. 

A common approach to address the curse of dimensionality is to impose sparsity constraints, where only a fraction of actions yield non-zero rewards.
\citet{agarwal2024multi} studied stochastic MABs under network interference, modeling rewards over the hypercube $[k]^N$ with each unit’s reward influenced by its own and its neighbors' arms.
Their regression-based algorithm achieves regret $\tilde{\mathcal{O}} \left(T^{2/3} \right)$, which is worse than the classical MAB settings, highlighting the impact of interference.
They further proposed a sequential action elimination algorithm with regret scaling as $\sqrt{T}$, albeit growing with $N$.
In this paper, we use the neighborhood structure to establish an upper bound on regret that improves over previous work by a factor between $\sqrt{N}$ and $N$, depending on the interference graph.
In particular, we avoid regrets growing with $N$, in contrast to previous work.
Furthermore, we derive matching lower bounds and prove that our algorithm is near-optimal for both sparse and dense graphs.

\section{Problem Setup}\label{sec: background}

We consider a stochastic multi-armed bandit setting with $N$ units and \( k \) treatment arms (or simply, arms).
At each round \( t \in [T] \), the learner selects an assignment \( A_t \in [k]^N \), where $A_{ti}$ denotes the treatment assigned to unit $i \in [N]$.
The reward for unit $i$ is given by \( Y_{ti}: [k]^N \to \mathbb R \), which depends on both the unit's own arm and those of its neighbors due to interference.

Interference is modeled by a graph \( \mathcal{G} \coloneqq ( [N], \mathcal{E}) \), where nodes represent units, and an edge \( (i,j) \in \mathcal{E} \) indicates that treatments assigned to $i$ and $j$ affect each other's rewards.
The neighborhood of unit \( i \) is denoted by \( N(i) \), which includes $i$ and $i$ and its neighbors, with size $|N(i)|= d_i+1$. 

At each round $t$, the learner selects a treatment assignment $A_t$ based on past observations and observes the rewards $Y_{ti}(A_t)$ \footnote{We assume that for each treatment $A$ and unit $i$, the reward distribution $Y_i(A)$ is $1$-sub-Gaussian, a standard assumption in the bandit literature \citep{lattimore2020bandit, bubeck2012regret}. } for all $i \in [N]$. 

Formally, a policy $\pi \coloneqq (\pi_1, \dots, \pi_T)$ is a sequence of mappings:
\[
\pi_t : \big([k]^N \times \mathbb R^N\big)^{t-1} \to \mathcal{P}([k]^N),
\]
where \( \mathcal{P}([k]^N) \) denotes the space of probability distributions over treatment assignments and the domain represents the history of past assignments and rewards.
The treatment at round $t$ is drawn as:
\[
A_t \sim \pi_t\left(A_1, Y_1, \ldots, A_{t-1}, Y_{t-1}\right).
\]
To evaluate a policy over $T$ rounds, we define cumulative regret as the gap between the optimal reward in hindsight and the expected reward achieved by the policy. 
The expected reward for unit $i$ under treatment $A$ is given by 
\begin{equation*}
    \begin{aligned}
        \mu_i(A) \coloneqq& \mathbbm E[Y_i| A_{N(i)}] ,\\ \quad \text{with} \quad\mu_i(A) &\in [0,1] \quad \text{for all} \quad i \in [N] , A \in [k]^N.
    \end{aligned}
\end{equation*}

\begin{definition}[Regret]\label{def: reg}
The regret of policy $\pi$ on instance $\V$ is defined as:
\begin{align*}
   & {Reg}_T(\pi, \V)  \coloneqq \\ &\frac{1}{N} \mathbb{E}_{A_t\sim\pi} \Big[\max_{A \in [k]^N} T \sum_{i \in [N]} \mu_i(A) - \sum_{\substack{i \in [N] \\ t \in [T]}}\mu_i(A_t)\Big],
\end{align*}
where the first term is the cumulative reward under an optimal treatment assignment, and the second corresponds to the cumulative reward achieved by the policy $\pi$.
The instance $\V$ includes the interference graph and the reward distributions.
For simplicity, we may omit $\pi$ or $\V$ when clear from context. 
\end{definition}

\begin{remark}
If the interference graph $\G$ is not connected, each connected component can be treated as an independent instance.
The problem can then be solved separately for each component, and the total regret is the weighted average of regrets over them.
The same applies to proving lower bounds.
Hence, we assume $\G$ is connected throughout the paper.
\end{remark}

\section{Upper bound}\label{sec: upper}

In this section, we propose the Partitioned UCB with Interference (PUCB-I) algorithm to address the MAB problem under network interference.
The algorithm relies on a partitioning of units, based on a structural property of the interference graph, to choose its next treatment assignment. 

Specifically, we define a binary relation over pairs of units $i, i' \in [N]$, where $i \sim i'$ if they are connected and have identical neighborhoods in the interference graph $\G$; that is,
\begin{align} \label{eq: relation}
    i \sim i' \Longleftrightarrow  N(i) =  N(i').
\end{align}

This relation is an equivalence, partitioning $[N]$ into $M$ classes $\{P_1, P_2, \dots, P_M\}$.
Let $m_j = |P_j|$ denote the size of partition $P_j$ and $D_j$ the common degree of its nodes, since they share the same neighborhood.
Figure \ref{fig: partition} depicts such a partitioning with $P_1 = \{A, B\}$, $P_2 = \{C\}$,  $P_3 = \{D, E\}$, and $P_4 = \{F, G, H\}$.
For a graph $\G$, we denote by $M(\G)$ the number of partitions induced by the relation in Equation \eqref{eq: relation} and use $M$ when $\G$ is clear from context.

To explain the key components of the algorithm, we first introduce two essential definitions on graphs.

\begin{definition}[Doubly-Independent Set] \label{def: DI}
    Let $\G$ be a connected graph with $N$ nodes.
    A set of nodes $S$ is called a \textit{doubly-independent set} if no two nodes in $S$ are adjacent, nor do they share a common neighbor outside $S$. 
    The family of all doubly-independent sets of $\G$ is denoted by $DI(\G)$.
\end{definition}

\begin{definition}[Square Chromatic Number] \label{def: square}
    The \textit{square chromatic number} of a graph $\G$ is the minimum number of colors required to color its nodes such that all nodes of the same color form a doubly-independent set.
    The square chromatic number of $\G$ is denoted by $\chi(\G^2)$.
\end{definition}

\begin{figure}[t]
    \centering
    \begin{tikzpicture}[
        scale=0.8, 
        node/.style={
            circle, 
            draw, 
            thick, 
            minimum size=0.3cm, 
            inner sep=2pt,
            font=\small\sffamily\bfseries
        },
        P1/.style={fill=red!35!white},
        P2/.style={fill=blue!35!white},
        P3/.style={fill=green!40},
        P4/.style={fill=yellow!40},
        edge/.style={-},
    ]
    \node[node, P1] (A) at (-3, 1) {A};
    \node[node, P1] (B) at (-3, -1) {B};
    
    \node[node, P4] (C) at (-1, 0) {C};
    
    \node[node, P2] (D) at (1, 2) {D};
    \node[node, P2] (E) at (1, -2) {E};
    
    \node[node, P3] (F) at (2, 0) {F};
    \node[node, P3] (G) at (4, 1) {G};
    \node[node, P3] (H) at (4, -1) {H};

    \draw[black] (A) -- (B);
    \draw[black] (A) -- (C);
    \draw[black] (B) -- (C);
    
    \draw[black] (C) -- (D);
    \draw[black] (C) -- (E);

    \draw[black] (D) -- (E);
    \draw[black] (D) -- (F);
    \draw[black] (D) -- (G);
    \draw[black] (D) -- (H);
    \draw[black] (E) -- (F);
    \draw[black] (E) -- (G);
    \draw[black] (E) -- (H);

    \draw[black] (F) -- (G);
    \draw[black] (F) -- (H);
    \draw[black] (G) -- (H);

    \end{tikzpicture}
    \caption{
    Graph partitions from Eq.~\eqref{eq: relation}; same-color nodes share a partition.}
    \label{fig: partition}
    \vspace{-0.3cm}
\end{figure}

\paragraph{Initialization.}
In this phase, the algorithm aims to collect at least one sample from each possible treatment $A_{N(i)}$ for every unit $i$ and its neighbors. Exhaustively exploring all $k^N$ joint treatments is highly inefficient, as it results in collecting excessive samples when the degrees are small, leading to high regret during initialization.

To address this, we propose a graph-based initialization strategy that efficiently explores the relevant treatment space in a small number of rounds. First, we partition the graph nodes into sets $S_1, S_2, \ldots, S_{\chi(\G^2)}$ according to Definition \ref{def: square}, where each $S_l$ is a doubly-independent set (note that these partitions differ from those defined by Relation \ref{eq: relation}). Let $n_l$ denote the maximum degree of nodes in $S_l$.

For each $S_l$, we allocate $k^{n_l+1}$ rounds to ensure that each unit $i \in S_l$ receives at least one sample for every configuration of $A_{N(i)}$. Due to the doubly-independent property, the treatment assignments for one unit and its neighbors do not interfere with the assignments of other units in the same set or their neighbors. This enables us to simultaneously explore all treatment combinations for all units in $S_l$ and their neighborhoods.
By devoting $k^{n_l+1}$ rounds to each $S_l$, we assign treatments in a symmetric manner to ensure that every unit $i \in S_l$ receives exactly $k^{n_l - d_i}$ samples from each possible configuration of $A_{N(i)}$. Since $n_l \leq \Delta$ for all $l$, the total number of rounds required for initialization is at most:  
$$
\sum_{l \in [\chi(\G^2)]} k^{n_l+1} \leq \chi(\G^2) \cdot k^{\Delta+1}.
$$
This method is significantly more efficient than exhaustive enumeration over all $k^N$ treatments or even taking one round for each treatment configuration per unit and its neighbors, which would require $\sum_{i \in [N]} k^{d_i+1}$ rounds.
We will later show that $\chi(\G^2)$ can be upper bounded in terms of $\Delta$; see Lemma~\ref{lem: constant-chromatic}.

\begin{algorithm}[tb]
   \caption{Partitioned UCB with Interference (PUCB-I)}
   \label{alg:ucb_bandit}
\begin{algorithmic}[1]
   \STATE {\bfseries Input:} Number of rounds \( T \), interference graph \( \mathcal{G} \), number of arms \( k \), confidence parameter \( \delta \).
   \STATE {\bfseries Initialization:} Perform the initialization phase described above and set $T_0 = \sum_{l \in [\chi(\G^2)]} k^{n_l+1}$.
   \FOR{each round \( t =  T_0, \dots, T-1 \)}
       \FOR{each partition \( P_j \in \{P_1, \dots, P_M\} \)}
           \STATE Compute \( \text{UCB}_{tP_j}(A) \) for all \( A \in [k]^N \) using Eq.~\eqref{eq: ucb}.
       \ENDFOR
       \STATE Select $A_{t+1}$ using Eq.~\eqref{eq: At}.
       \STATE Observe rewards \( Y_{t+1i}(A_{t+1}) \) for all \( i \in [N] \).
       \STATE Update \( \hat{\mu}_{t+1i}(A) \) and \( n_{t+1P_j}(A) \) for the explored treatments.
   \ENDFOR
\end{algorithmic}
\end{algorithm}
\paragraph{Empirical Mean Reward.} For each unit \( i \in [N] \) and treatment $ A_{N(i)} \in [k]^{d_i+1} $, the empirical mean reward is estimated as:
  \begin{equation*}
  \begin{aligned}
            \hat{\mu}_{ti}(A) &= \frac{\sum_{t'=1}^t Y_{t'i}(A_{N(i)})}{n_{ti}(A)} \quad\\ \text{with}  \quad  n_{ti}(A) = &\sum_{t'=1}^t \mathbbm{1} \{ A_{t'N(i)}= A_{N(i)}\},
  \end{aligned}
  \end{equation*} 
where \( n_{ti}(A) \) is the number of times $A_{N(i)}$ has been assigned to unit $i$ up to time \( t \).
Since $n_{ti}(A) = n_{ti'}(A)$ for all $i, i' \in P_j$, we define $\npj \coloneqq n_{ti}(A)$ for any unit $i \in P_j$.

\textbf{UCB.} For each partition \( P_j\) at round $t$, compute the UCB for every treatment \( A \) as:
    \begin{equation}\label{eq: ucb}
             UCB_{tP_j} (A) \coloneqq \sum_{i \in P_j}\hat{\mu}_{ti}(A) + \sqrt{2 \log( 2/\delta) \frac{m_j}{\npj}}.
    \end{equation}
\textbf{Treatment Assignment.} 
At each round, the learner selects the treatment that maximizes the sum of UCBs across all partitions:
\begin{equation}\label{eq: At}
            A_{t+1} =\argmax_{A \in [k]^N} \sum_{j \in [M]} UCB_{tj}(A) .
\end{equation}

Following standard practice in the combinatorial bandit literature \citep{cucb-chen2013combinatorial, cesa2012combinatorial, chen2016combinatorial, wang2017improving, perrault2020statistical, tzeng2023closing}, we assume access to a computationally efficient offline oracle that, given the current estimates (e.g., mean rewards or UCBs), returns the treatment assignment maximizing the objective in Equation~\eqref{eq: At}.
This abstraction, commonly referred to as a linear maximization (LM) oracle, is standard in the combinatorial bandits. Without such an oracle, even solving the offline optimization problem with known rewards is generally intractable, rendering regret minimization in the online setting computationally infeasible.

Regarding the existence of such an oracle, the recent work \citet{shahverdikondori2026neighborhood} studies the computational complexity of this optimization problem. The authors show that it is NP-hard for general graphs and propose a dynamic programming algorithm that computes the exact optimizer in time $O\!\left(k^{tw(\G^2)} N^2\right)$, where $tw(\G^2)$ denotes the \emph{treewidth} of the squared graph. They further prove that this running time is essentially optimal under standard complexity assumptions. Treewidth is a graph parameter that measures how “tree-like” a graph is; for example, every tree has treewidth $1$ (see \citet{cygan2015parameterized}). 
This result highlights that the optimization is computationally tractable when $tw(\G^2)$ is small.
In addition, they show that when rewards are positive, efficient constant-factor approximation algorithms exist for classes of practical graphs, such as those with constant maximum degree $\Delta$.

 \textbf{Reward Observation and Update.} After selecting $A_{t+1}$, the algorithm observes \( Y_{t+1i}(A_{t+1}) \) for all \( i \in [N] \), and updates \( \hat{\mu}_{t+1i}(A) \) and \( n_{t+1P_j}(A) \) accordingly.

The following theorem establishes a graph-dependent upper bound on the expected cumulative regret of Algorithm \ref{alg:ucb_bandit}:

\begin{restatable}{theorem}{thupper}[Graph-Partitioned Regret Upper Bound]\label{thm: upper-bound}
The expected cumulative regret of Algorithm \ref{alg:ucb_bandit} with $\delta = (T^2N \sum_{j \in [M]} k^{D_j+1})^{-1}$ interacting with any instance with $1$-sub-Gaussian rewards and interference graph $\G$, partitioned into $P_1, P_2, \ldots, P_M$, satisfies:
\[
{Reg}_T  \in  \mathcal{O}\left( \sqrt{ \frac{T}{N^2} \log (TN)} \sum_{j \in [M]}  \sqrt{m_j k^{D_{j}+1}}\right).
\]
\end{restatable}

The dependency on \( k^{D_j+1} \) reflects the complexity introduced by interference through each partition’s neighborhood size.
A perhaps more interpretable form of the bound can be obtained in terms of the maximum degree $\Delta:= \max_{i\in [N]} d_i$:
\begin{corollary}\label{cor:delta}
    Let $\Delta$ denote the maximum degree of $\mathcal{G}$. 
    Then,
 \begin{equation*}
     {Reg}_T \in \mathcal{O} \left( \sqrt{\frac{TM}{N} k^{\Delta+1} \log (TN)}\right ).
 \end{equation*}
\end{corollary}
This follows from Cauchy-Schwarz inequality:

\begin{align*}
    \sum_{j \in [M]}  \sqrt{m_j k^{\Delta+1}} &\leq \sqrt{ M \sum_{j\in [M]} m_jk^{\Delta+1}}\\&= \sqrt{ M N k^{\Delta+1}}.
\end{align*}

Corollary \ref{cor:delta} allows comparison to the state-of-the-art. 
\citet{agarwal2024multi} derived a regret upper bound of $$\tilde{\mathcal{O}} \left(\sqrt{TNk^{\Delta+ 1}}\right),$$ which is a factor $N /\sqrt{M}$ worse than PUCB-I. 
At one end, when $M = N$ (i.e., each unit forms its own partition), the gap is $\sqrt{N}$, while at the other end, when $M = 1$, it can be as large as $N$.

\begin{remark}\label{remark: comb}
    Compared to combinatorial bandits, our setting can be regarded as a special case with $\sum_{i \in [N]} k^{d_i+1}$ base arms and $k^N$ super arms where each super arm consists of $N$ arms and represents a joint treatment assignment to all $N$ units.
    At each round, the learner selects one super arm and observes outcomes for all individual units. In terms of feedback structure, this aligns with the \emph{semi-bandit} model studied in the combinatorial bandit literature, where the learner observes individual rewards for each selected base arm.
    For this setting, the tight lower bound in the combinatorial bandit literature \citep{cesa2012combinatorial, audibert2014regret, merlis2020tight} is $${\Omega} \left( \sqrt{\frac{T}{N} \sum_{i \in [N]}k^{d_i+1}} \right),$$
    which can be larger than our upper bound in Corollary \ref{cor:delta}. For example, on regular graphs where all nodes have the same degree, the gap is $\sqrt{\frac{N}{M}}$.
    This gap indicates that algorithms optimal for the general combinatorial bandit problem are not optimal in our setting. As a result, specialized algorithms are required to exploit the structure induced by network interference.
\end{remark}

If partitioning is ignored and UCBs are computed per unit, the regret bound becomes
$$ 
\mathcal{O}\left( \sqrt{ \frac{T}{N^2} \log (TN)} \sum_{i \in [N]}  \sqrt{k^{d_{i}+1}}\right).
$$  
This bound is always at least as large as the bound in Theorem \ref{thm: upper-bound} since
$$
\sum_{i \in [N]}  \sqrt{k^{d_{i}+1}} = \sum_{j \in [M]} m_j\sqrt{k^{D_j+1}} \geq \sum_{j \in [M]}  \sqrt{m_j k^{D_{j}+1}}.
$$  
This indicates that partitioning reduces the total regret from partition $P_j$ by a factor of $\sqrt{m_j}$. For example, in the case of a complete graph, this yields a $\sqrt{N}$ improvement.

\section{Lower Bounds}\label{sec: lower}

In this section, we establish the first lower bounds on the expected regret for MABs with arbitrary network interference.
We derive two distinct bounds that quantify the worst-case regret of any algorithm based on the topological properties of the underlying interference graph. 
Specifically, these properties pertain to i) the symmetries in neighborhoods of the nodes and ii) the structure of doubly-independent sets of the interference graph.
Proofs are provided in Appendix \ref{sec: proof lower}.
    
Both lower bounds have a gap with the proposed upper bound for PUCB-I in Theorem \ref{thm: upper-bound}.
Subsequently, we identify classes of graphs where these gaps are constant, showing that our algorithm is nearly optimal (up to logarithmic factors). 
The first lower bound demonstrates that PUCB-I is near-optimal for classes of dense graphs, while the second proves that it is near-optimal for sparse graphs.
Sparse graphs are particularly significant in multi-armed bandit problems with interference, as many practical applications involve interference graphs that are sparse due to limited local interactions \citep{agarwal2024multi, yang2016learning}.
The following theorem establishes our first lower bound on the expected regret.



\begin{restatable}{theorem}{theoremlowerbound} [Graph-Partitioned Regret Lower Bound] \label{them: general-lower-bound}
    Let $\G$ be a connected graph with $N$ nodes and partitions $\left\{P_1, P_2, \ldots, P_M\right\}$ with $|P_j| = m_j$.
    If $k > \frac{2^{\frac{1}{\Delta+1}}}{2^{\frac{1}{\Delta+1}} - 1}$ and $T \geq \frac{4 (k-1)^{\Delta + 1}}{M}$, then for any policy $\pi$, there exists a bandit instance with interference graph $\G$ whose  reward is distributed as $1$-Gaussian with means in $[0,1]$ such that
    \begin{align*}
        {Reg}_T(\pi) \in \Omega \left( \sqrt{\frac{T}{N^2M}} \sum_{j \in [M]} \sqrt{m_jk^{D_j+1}} \right).
    \end{align*}
\end{restatable}

This lower bound has a gap of $\sqrt{M}$ compared to the upper bound of PUCB-I. 
This indicates that the algorithm achieves better performance for graphs with a smaller number of partitions.

We now derive a second lower bound on the expected regret that nearly matches the upper bound of our proposed algorithm for graphs with a bounded maximum degree.

\begin{restatable}{theorem}{theoremsparselowerbound}
[Doubly-Independent Set Regret Lower Bound] \label{them: sparse-lower-bound}
    Let $\G$ be a connected graph with $N$ nodes. If $T \geq k^{\Delta + 1} - 1$, for any policy $\pi$, there exists a bandit instance with interference graph $\G$, $1$-Gaussian reward distributions and means in $[0,1]$ such that
    \begin{align} \label{eq: sparse-lower-bound}
        {Reg}_T(\pi) \in \Omega \left( \max_{S \in DI(\G)} \sqrt{\frac{T}{N^2}} \sum_{i \in S} \sqrt{k^{d_i + 1}} \right),
    \end{align}
    where $DI(\G)$ is defined in Definition \ref{def: DI}.
\end{restatable}

The following corollary relates doubly-independent sets to the square chromatic number $\chi(\G^2)$, allowing comparison between the two lower bounds.

\begin{corollary}
    A connected graph $\G$  can be colored with $\chi(\G^2)$ distinct colors, where each color class is a doubly-independent set. 
    Therefore, the maximum in \eqref{eq: sparse-lower-bound} is at least equal to $\frac{1}{\chi(\G^2)}$ times the sum on all the nodes, implying
    \begin{align} \label{eq: scn-lower-bound}
         {Reg}_T(\pi) \in \Omega \left( \frac{1}{\chi(\G^2)} \sqrt{\frac{T}{N^2}} \sum_{i \in [N]} \sqrt{k^{d_i+1}} \right).
    \end{align}
\end{corollary}

The first lower bound has a $\sqrt{M(\G)}$ gap with the upper bound; the second has a gap of $\chi(\G^2)$.
Thus, the tighter bound depends on whether $\chi(\G^2) > \sqrt{M(\G)}$ or not.
Overall, the gap is at most $$\min\left(\sqrt{M(\G)}, \chi(\G^2)\right) $$ which is always bounded by $\sqrt{N}$.
Therefore, for graphs with constant $\sqrt{M(\G)}$ or $\chi(\G^2)$, PUCB-I achieves near-optimal regret.

\subsection{Graphs with Tight Bounds}
This section explores the classes of graphs where $\min\left(\sqrt{M(\G)}, \chi(\G^2)\right)$ is constant, yielding tight regret bounds.

To characterize the graphs for which the first lower bound shows the near-optimality of our algorithm, i.e., where $\sqrt{M(\G)}$ is constant, we define a class of dense graphs, called \textit{Clique-Sparse Graphs}.

\begin{definition}[Clique-Sparse Graph]
    A graph $\G$ with $N$ nodes is $(R,r)$-\textit{Clique-Sparse} if its nodes can be partitioned into $R$ clusters $C_1, C_2, \ldots, C_R$ such that: 
    \begin{itemize}
        \item Each cluster $C_i$ forms a complete graph, and
        \item There are at most $r$ edges between any pair of clusters.
    \end{itemize}
\end{definition}

\begin{figure}[t]
    \centering
    \begin{tikzpicture}[
        scale=0.9, 
        node/.style={
            circle, 
            draw, 
            thick, 
            minimum size=0.3cm, 
            inner sep=2pt,
            font=\small\sffamily\bfseries
        },
        clique1/.style={fill=red!40},
        clique2/.style={fill=blue!40},
        clique3/.style={fill=green!40},
        edge/.style={-},
    ]
    \node[node, clique1] (A) at (-0.8, -0.2) {A};
    \node[node, clique1] (B) at (-1.6, -1.2) {B};
    \node[node, clique1] (C) at (0, -1.2) {C};
    \node[node, clique1] (D) at (-0.8, -2.2) {D};
    
    \node[node, clique2] (E) at (3.8, -0.2) {E};
    \node[node, clique2] (F) at (2.9, -1.2) {F};
    \node[node, clique2] (G) at (4.6, -1.2) {G};
    \node[node, clique2] (H) at (3.8, -2.2) {H};
    
    \node[node, clique3] (I) at (1.3, 1.6) {I};
    \node[node, clique3] (J) at (0.5, 0.6) {J};
    \node[node, clique3] (K) at (2.1, 0.6) {K};
    
    \foreach \u/\v in {A/B, A/C, A/D, B/C, B/D, C/D}
        \draw[black] (\u) -- (\v);
    
    \foreach \u/\v in {E/F, E/G, E/H, F/G, F/H, G/H}
        \draw[black] (\u) -- (\v);
    
    \foreach \u/\v in {I/J, I/K, J/K}
        \draw[black] (\u) -- (\v);
    
    \draw[black] (A) -- (E);
    \draw[black] (D) -- (F);
    
    \draw[black] (C) -- (J);
    
    \draw[black] (F) -- (J);
    \draw[black] (F) -- (K);
    
    \end{tikzpicture}
    \caption{ A $(3,2)$-clique-sparse graph with colored clusters.}
    \label{fig: clique-sparse}
\end{figure}

For example, Figure~\ref{fig: clique-sparse} shows a $(3,2)$-clique-sparse graph, while a complete graph is $(1,0)$-clique-sparse.

The next lemma shows that if $\G$ is $(R,r)$-clique-sparse with constant $R$ and $r$, then $M(\G)$ is also constant, establishing that PUCB-I is near-optimal for this class of graphs.

 
\begin{restatable}{lemma}{cliquesparse}\label{lem: clique-sparse}
        For a $(R,r)$-clique-sparse graph $\G$, the number of partitions $M(\G)$ induced by the equivalence relation in Equation \eqref{eq: relation} satisfies:
    \begin{align*}
        M(\G) \leq R + rR(R-1).
    \end{align*}
\end{restatable}

The following lemma shows that for graphs with constant maximum degree $\Delta$, the square chromatic number is also constant, implying that PUCB-I is near optimal for these graphs as well.

\begin{restatable}{lemma}{constantchrom} \label{lem: constant-chromatic}
    For any graph $\G$ with maximum degree $\Delta$, the square chromatic number satisfies $$\chi(\G^2) \leq \Delta^2 + 1.$$
\end{restatable}

\section{Unknown Interference Graph}\label{sec: unkown}
The assumption of knowing the interference graph may be unrealistic in certain scenarios. In this section, we study regret minimization when the graph is unknown.
We show that, without any prior information on the graph, no algorithm can uniformly outperform the one that assumes a complete graph.
For a policy $\pi$, graph $\G$, we denote by $\rtgpi{\G}$ the worst-case regret of policy $\pi$ over all instances with 1-sub-Gaussian rewards and true graph $\G$ when the graph is unknown.

A natural extension of PUCB-I to this setting is Partitioned UCB with Unknown Interference (PUCB-UI) which assumes that for any instance, the interference graph is complete and runs PUCB-I accordingly.
Clearly, this algorithm satisfies: 
$$
\forall \G: R_T(\pi_{\text{PUCB-UI}}, \G) \in \mathcal{O}\left( \sqrt{\frac{T}{N} k^N} \right),
$$  
matching the worst-case regret of PUCB-I with a known, complete graph.

The following theorem shows that no algorithm can uniformly outperform PUCB-UI across all problem instances.
First, we define a key concept.

\begin{definition}[Legit Policy] 
    A policy $\pi$ is \textit{legit} if its worst-case regret (under unknown interference) occurs when the true graph is complete, i.e.,
    \begin{align*}
        \max_{\G} \rtgpi{\G} = \rtgpi{K_N},
    \end{align*}
    where $K_N$ denotes complete graph on $N$ nodes. 
\end{definition}

\begin{restatable}{theorem}{unknownLB}[Unknown Graph Lower Bound] \label{thm: unknown-LB}
    For any legit policy $\pi$ and graph $\G$, if $k > \frac{2^{\frac{1}{N}}}{2^{\frac{1}{N}} - 1}$,
    \begin{align*}
        \rtgpi{\G} \rtgpi{K_N} \in \Omega \left( \frac{T}{N} k^N \right). 
    \end{align*}
\end{restatable}
The proof appears in Appendix \ref{sec: proofs uknown}.

\begin{corollary}
    Theorem \ref{thm: unknown-LB} implies that for a policy $\pi$, if there exists a constant $c_1$ such that $$\forall \G: \rtgpi{\G} \leq c_1 \sqrt{\frac{T}{N}k^N},$$ then there exists a constant $c_2$ such that $$\forall \G: \rtgpi{\G} \geq c_2 \sqrt{\frac{T}{N}k^N}.$$ Thus, no algorithm can uniformly outperform PUCB-UI across all instances.
\end{corollary}

\section{Experiments}\label{sec: exp}
We empirically evaluate PUCB-I (Algorithm~\ref{alg:ucb_bandit}) against the following baselines in simulations to validate our theoretical guarantees.

\textbf{Classical UCB.} This algorithm ignores the interference graph, treats each treatment in $k^{[N]}$ as an independent arm, and performs the UCB algorithm.  

\textbf{Combinatorial UCB (CUCB).} Proposed by \cite{cucb-chen2013combinatorial}, CUCB is one of the most widely used algorithms in the combinatorial bandit framework. By formulating our problem within this framework, CUCB can be applied to the MAB setting with interference. 
    
\textbf{Network Explore-Then-Commit (ETC).} Proposed by \citet{agarwal2024multi}, this algorithm runs in two phases: it first assigns treatments uniformly at random, then estimates the reward parameters via least squares regression, using the known interference graph.
In the second phase, it plays the arm with the highest estimated reward for the remaining rounds.

\textbf{Sequential Action Elimination (SAE).} Also from \citet{agarwal2024multi}, this elimination-based algorithm begins with all $k^{[N]}$ arms active. In each epoch, it pulls each active arm equally and eliminates those with poor observed performance.  

The complete implementation details are provided in the supplementary material.  
In all experiments, we set $k = 2$ and $\Delta = 3$. The total number of rounds is set to $T_{\text{max}} = 10 \cdot 2^N$, ensuring sufficient exploration of all $2^N$ arms, as required by classical UCB and SAE algorithms.  
The rewards for each unit are independently drawn from a $1$-Gaussian distribution. Additionally, the theoretical value of $\delta$ suggested in Theorem \ref{thm: upper-bound} is overly conservative in practice. 
We therefore used the fixed value $\delta=0.01$ in all experiments.
We conducted two sets of experiments, analyzing the cumulative regret as a function of $T$ and $N$. All results are averaged over $50$ independent runs, with the standard deviation of the regret displayed in the plots.

\begin{figure}[h]
    \centering
    \includegraphics[width=\linewidth]{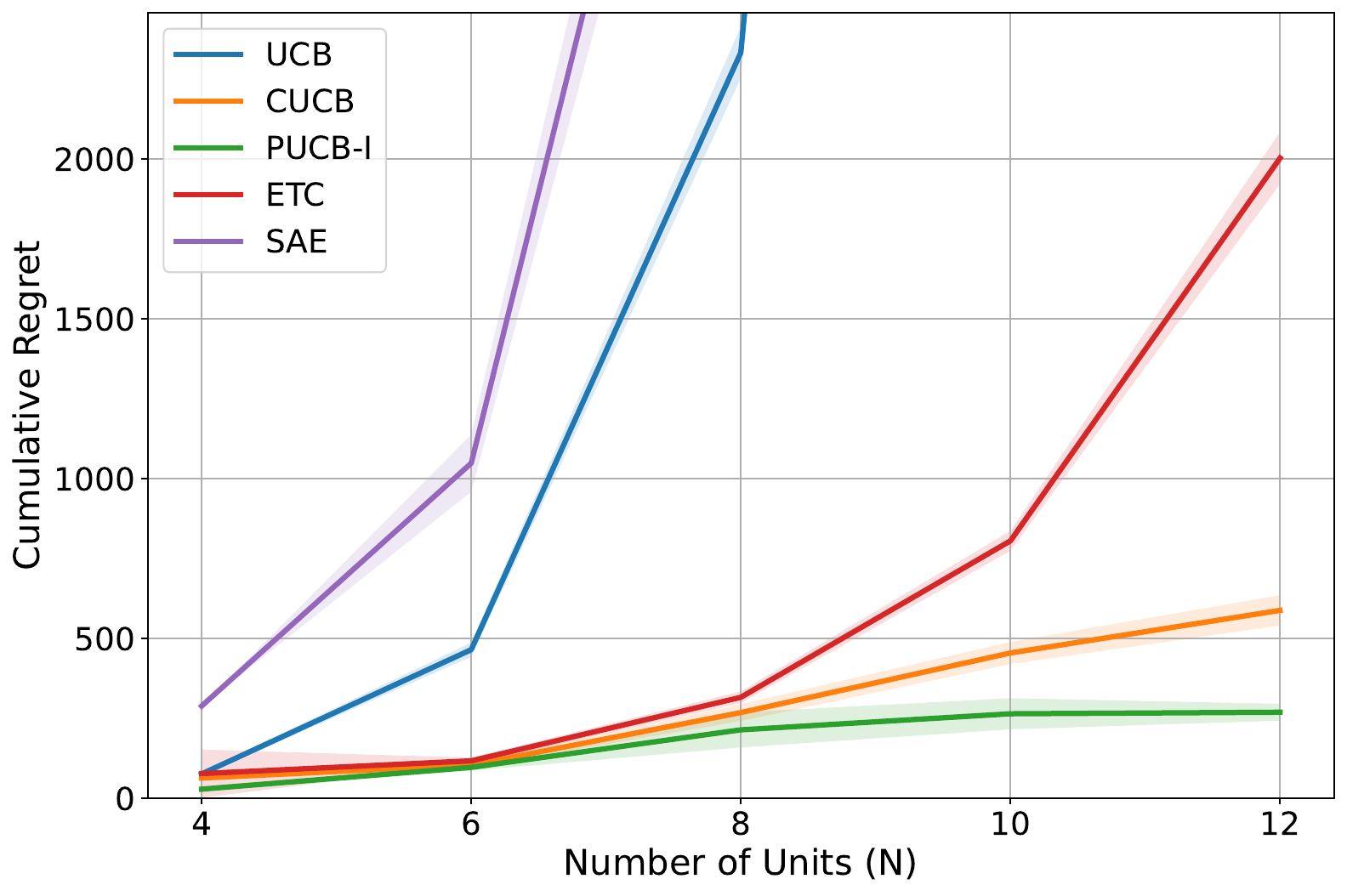}
    \caption{Average regret vs. number of units ($N$).}
    \label{fig:N}
\end{figure}

\paragraph{Scaling of Regret with $T$.}
Figures \ref{fig:T10} and \ref{fig:T12} present the cumulative regret of different algorithms over time for $N = 10$ and $N = 12$ units. The results highlight the superior performance of PUCB-I in both settings. In comparison, UCB and SAE exhibit a linear increase in regret during their long exploration phases, resulting in significantly higher overall regret and lower practical effectiveness.


\begin{figure}[h]
    \centering
    \begin{subfigure}[h]{\linewidth}
        \centering
        \includegraphics[width=\linewidth]{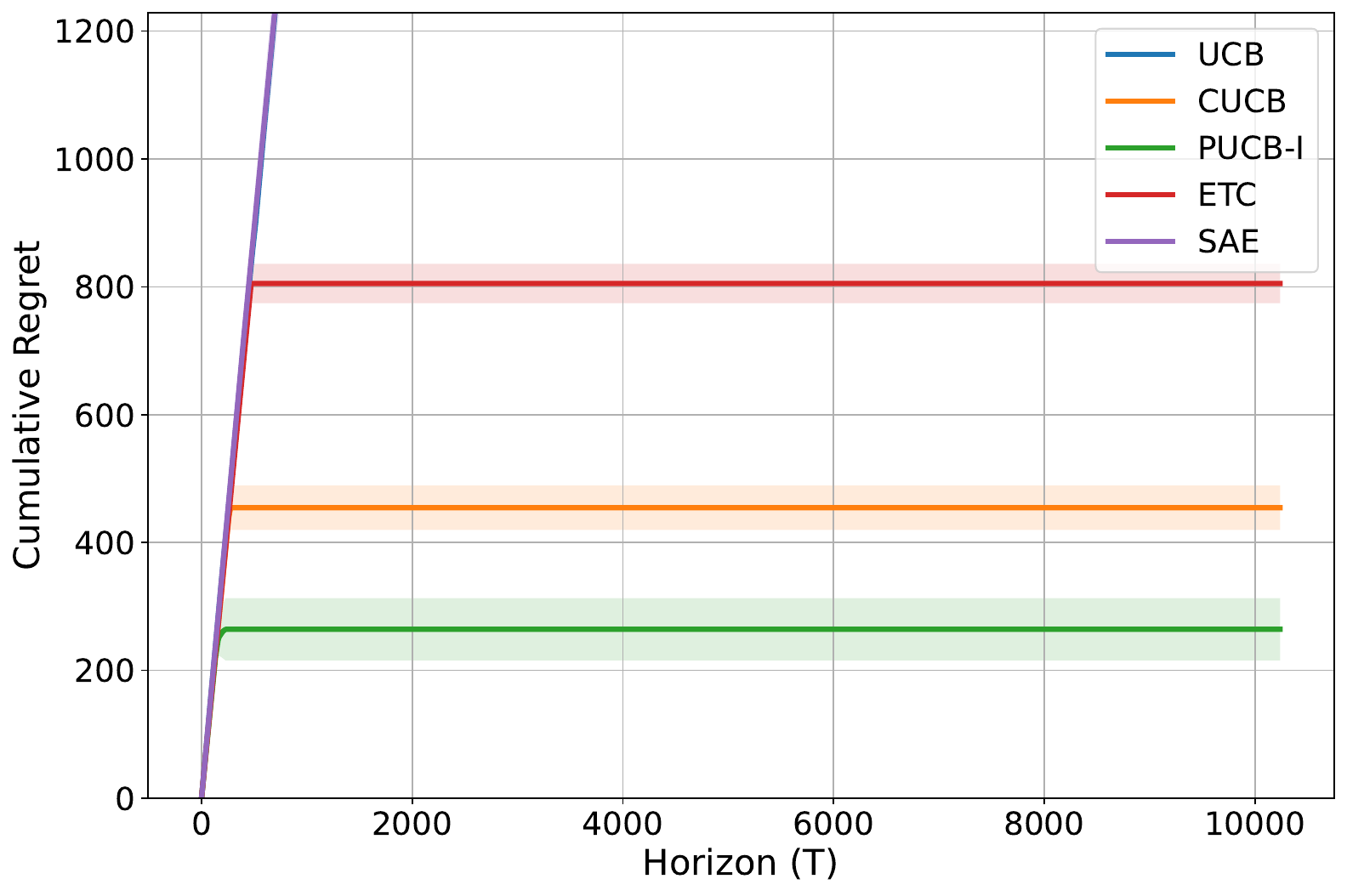}
        \caption{Average regret over time for an instance with \\$N=10, k=2$ and $\Delta = 3$.}
        \label{fig:T10}
    \end{subfigure}
    
    \begin{subfigure}[h]{\linewidth}
        \centering
        \includegraphics[width=\linewidth]{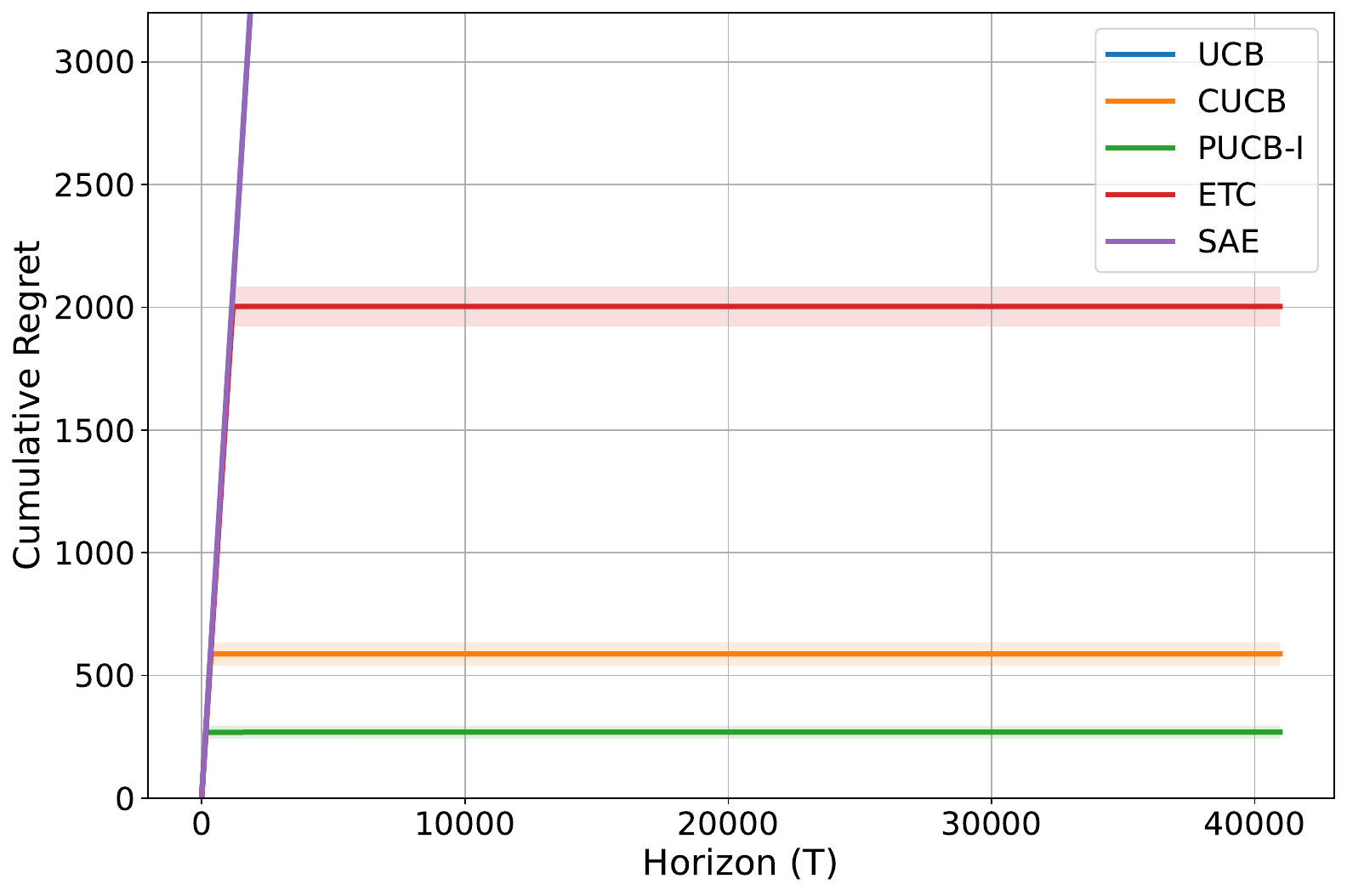}
        \caption{Average regret over time for an instance with \\$N=12, k=2$ and $\Delta = 3$.}
        \label{fig:T12}
    \end{subfigure}
    \caption{Comparison of average regret for various algorithms.}
    \label{fig:combined-regret}
\end{figure}

\paragraph{Scaling of Regret with $N$.}
To evaluate the impact of the number of units $N$, we analyzed the cumulative regret of the algorithms across instances by varying $N$ in $\{4, 6, 8, 10, 12\}$. 
As shown in Figure \ref{fig:N}, when $N$ increases, the regret of UCB grows exponentially, confirming that it was not able to incorporate the graph structure. 
In contrast, the regret of PUCB-I grows more slowly as $N$ increases compared to all other baselines because it benefits from partitioning the units with similar neighborhoods and graph-based initialization.


\section{Conclusion}
We studied the MAB problem with network interference, where each unit's reward depends on its own treatment and those of its neighbors. 
To address the resulting challenges posed by interference, we proposed the PUCB-I algorithm, which partitions units based on their neighborhood structure to minimize cumulative regret.
We derived a graph-dependent regret upper bound and established the first regret lower bounds for bandits with arbitrary network interference, showing that our proposed algorithm is near-optimal (up to logarithmic factors) for both sparse and dense graphs.
We also considered the setting where the interference graph is unknown and showed that a variant of our algorithm is Pareto optimal.
Our work highlights the importance of accounting for the graph structure in sequential decision-making under interference.

\bibliography{biblio}

\begin{thebibliography}{53}
\providecommand{\natexlab}[1]{#1}
\providecommand{\url}[1]{\texttt{#1}}
\expandafter\ifx\csname urlstyle\endcsname\relax
  \providecommand{\doi}[1]{doi: #1}\else
  \providecommand{\doi}{doi: \begingroup \urlstyle{rm}\Url}\fi

\bibitem[Adhikari et~al.(2025)Adhikari, Medya, and Zheleva]{elena1-adhikari2025learning}
Shishir Adhikari, Sourav Medya, and Elena Zheleva.
\newblock Learning exposure mapping functions for inferring heterogeneous peer effects.
\newblock \emph{arXiv preprint arXiv:2503.01722}, 2025.

\bibitem[Agarwal et~al.(2024)Agarwal, Agarwal, Masoero, and Whitehouse]{agarwal2024multi}
Abhineet Agarwal, Anish Agarwal, Lorenzo Masoero, and Justin Whitehouse.
\newblock Multi-armed bandits with network interference.
\newblock \emph{arXiv preprint arXiv:2405.18621}, 2024.

\bibitem[Alon et~al.(2017)Alon, Cesa-Bianchi, Gentile, Mannor, Mansour, and Shamir]{alon2017nonstochastic}
Noga Alon, Nicolo Cesa-Bianchi, Claudio Gentile, Shie Mannor, Yishay Mansour, and Ohad Shamir.
\newblock Nonstochastic multi-armed bandits with graph-structured feedback.
\newblock \emph{SIAM Journal on Computing}, 46\penalty0 (6):\penalty0 1785--1826, 2017.

\bibitem[Anantharam et~al.(1987)Anantharam, Varaiya, and Walrand]{anantharam1987asymptotically}
Venkatachalam Anantharam, Pravin Varaiya, and Jean Walrand.
\newblock Asymptotically efficient allocation rules for the multiarmed bandit problem with multiple plays-part i: Iid rewards.
\newblock \emph{IEEE Transactions on Automatic Control}, 32\penalty0 (11):\penalty0 968--976, 1987.

\bibitem[Aronow(2012)]{aronow2012general}
Peter~M Aronow.
\newblock A general method for detecting interference between units in randomized experiments.
\newblock \emph{Sociological Methods \& Research}, 41\penalty0 (1):\penalty0 3--16, 2012.

\bibitem[Aronow and Samii(2017)]{aronow2017estimating}
Peter~M Aronow and Cyrus Samii.
\newblock Estimating average causal effects under general interference, with application to a social network experiment.
\newblock 2017.

\bibitem[Athey et~al.(2018)Athey, Eckles, and Imbens]{athey2018exact}
Susan Athey, Dean Eckles, and Guido~W Imbens.
\newblock Exact p-values for network interference.
\newblock \emph{Journal of the American Statistical Association}, 113\penalty0 (521):\penalty0 230--240, 2018.

\bibitem[Audibert et~al.(2014)Audibert, Bubeck, and Lugosi]{audibert2014regret}
Jean-Yves Audibert, S{\'e}bastien Bubeck, and G{\'a}bor Lugosi.
\newblock Regret in online combinatorial optimization.
\newblock \emph{Mathematics of Operations Research}, 39\penalty0 (1):\penalty0 31--45, 2014.

\bibitem[Auer(2002)]{auer2002finite}
P~Auer.
\newblock Finite-time analysis of the multiarmed bandit problem, 2002.

\bibitem[Bajari et~al.(2021)Bajari, Burdick, Imbens, Masoero, McQueen, Richardson, and Rosen]{bajari2021multiple}
Patrick Bajari, Brian Burdick, Guido~W Imbens, Lorenzo Masoero, James McQueen, Thomas Richardson, and Ido~M Rosen.
\newblock Multiple randomization designs.
\newblock \emph{arXiv preprint arXiv:2112.13495}, 2021.

\bibitem[Bajari et~al.(2023)Bajari, Burdick, Imbens, Masoero, McQueen, Richardson, and Rosen]{bajari2023experimental}
Patrick Bajari, Brian Burdick, Guido~W Imbens, Lorenzo Masoero, James McQueen, Thomas~S Richardson, and Ido~M Rosen.
\newblock Experimental design in marketplaces.
\newblock \emph{Statistical Science}, 38\penalty0 (3):\penalty0 458--476, 2023.

\bibitem[Basse and Feller(2018)]{basse2018analyzing}
Guillaume Basse and Avi Feller.
\newblock Analyzing two-stage experiments in the presence of interference.
\newblock \emph{Journal of the American Statistical Association}, 113\penalty0 (521):\penalty0 41--55, 2018.

\bibitem[Bhattacharya et~al.(2020)Bhattacharya, Malinsky, and Shpitser]{bhattacharya2020causal}
Rohit Bhattacharya, Daniel Malinsky, and Ilya Shpitser.
\newblock Causal inference under interference and network uncertainty.
\newblock In \emph{Uncertainty in Artificial Intelligence}, pages 1028--1038. PMLR, 2020.

\bibitem[Bouneffouf et~al.(2020)Bouneffouf, Rish, and Aggarwal]{mab1-bouneffouf2020survey}
Djallel Bouneffouf, Irina Rish, and Charu Aggarwal.
\newblock Survey on applications of multi-armed and contextual bandits.
\newblock In \emph{2020 IEEE Congress on Evolutionary Computation (CEC)}, pages 1--8. IEEE, 2020.

\bibitem[Bubeck et~al.(2009)Bubeck, Munos, and Stoltz]{pure-bubeck2009pure}
S{\'e}bastien Bubeck, R{\'e}mi Munos, and Gilles Stoltz.
\newblock Pure exploration in multi-armed bandits problems.
\newblock In \emph{Algorithmic Learning Theory: 20th International Conference, ALT 2009, Porto, Portugal, October 3-5, 2009. Proceedings 20}, pages 23--37. Springer, 2009.

\bibitem[Bubeck et~al.(2012)Bubeck, Cesa-Bianchi, et~al.]{bubeck2012regret}
S{\'e}bastien Bubeck, Nicolo Cesa-Bianchi, et~al.
\newblock Regret analysis of stochastic and nonstochastic multi-armed bandit problems.
\newblock \emph{Foundations and Trends{\textregistered} in Machine Learning}, 5\penalty0 (1):\penalty0 1--122, 2012.

\bibitem[Cesa-Bianchi and Lugosi(2012)]{cesa2012combinatorial}
Nicolo Cesa-Bianchi and G{\'a}bor Lugosi.
\newblock Combinatorial bandits.
\newblock \emph{Journal of Computer and System Sciences}, 78\penalty0 (5):\penalty0 1404--1422, 2012.

\bibitem[Chen et~al.(2014)Chen, Lin, King, Lyu, and Chen]{chen2014combinatorial}
Shouyuan Chen, Tian Lin, Irwin King, Michael~R Lyu, and Wei Chen.
\newblock Combinatorial pure exploration of multi-armed bandits.
\newblock \emph{Advances in neural information processing systems}, 27, 2014.

\bibitem[Chen et~al.(2013{\natexlab{a}})Chen, Wang, and Yuan]{chen2013combinatorial}
Wei Chen, Yajun Wang, and Yang Yuan.
\newblock Combinatorial multi-armed bandit: General framework and applications.
\newblock In \emph{International conference on machine learning}, pages 151--159. PMLR, 2013{\natexlab{a}}.

\bibitem[Chen et~al.(2013{\natexlab{b}})Chen, Wang, and Yuan]{cucb-chen2013combinatorial}
Wei Chen, Yajun Wang, and Yang Yuan.
\newblock Combinatorial multi-armed bandit: General framework and applications.
\newblock In \emph{International conference on machine learning}, pages 151--159. PMLR, 2013{\natexlab{b}}.

\bibitem[Chen et~al.(2016)Chen, Wang, Yuan, and Wang]{chen2016combinatorial}
Wei Chen, Yajun Wang, Yang Yuan, and Qinshi Wang.
\newblock Combinatorial multi-armed bandit and its extension to probabilistically triggered arms.
\newblock \emph{Journal of Machine Learning Research}, 17\penalty0 (50):\penalty0 1--33, 2016.

\bibitem[Cygan et~al.(2015)Cygan, Fomin, Kowalik, Lokshtanov, Marx, Pilipczuk, Pilipczuk, and Saurabh]{cygan2015parameterized}
Marek Cygan, Fedor~V. Fomin, {\L}ukasz Kowalik, Daniel Lokshtanov, D{\'a}niel Marx, Marcin Pilipczuk, Micha{\l} Pilipczuk, and Saket Saurabh.
\newblock \emph{Parameterized Algorithms}.
\newblock Springer, 2015.
\newblock \doi{10.1007/978-3-319-21275-3}.

\bibitem[Eckles et~al.(2017)Eckles, Karrer, and Ugander]{eckles2017design}
Dean Eckles, Brian Karrer, and Johan Ugander.
\newblock Design and analysis of experiments in networks: Reducing bias from interference.
\newblock \emph{Journal of Causal Inference}, 5\penalty0 (1):\penalty0 20150021, 2017.

\bibitem[Faruk and Zheleva(2025)]{elena2-faruk2025leveraging}
Ahmed~Sayeed Faruk and Elena Zheleva.
\newblock Leveraging heterogeneous spillover in maximizing contextual bandit rewards.
\newblock In \emph{Proceedings of the ACM on Web Conference 2025}, pages 3049--3060, 2025.

\bibitem[Gao and Ding(2023)]{gao2023causal}
Mengsi Gao and Peng Ding.
\newblock Causal inference in network experiments: regression-based analysis and design-based properties.
\newblock \emph{arXiv preprint arXiv:2309.07476}, 2023.

\bibitem[Hudgens and Halloran(2008)]{hudgens2008toward}
Michael~G Hudgens and M~Elizabeth Halloran.
\newblock Toward causal inference with interference.
\newblock \emph{Journal of the American Statistical Association}, 103\penalty0 (482):\penalty0 832--842, 2008.

\bibitem[Jamshidi et~al.(2024)Jamshidi, Etesami, and Kiyavash]{jamshidi2024confounded}
Fateme Jamshidi, Jalal Etesami, and Negar Kiyavash.
\newblock Confounded budgeted causal bandits.
\newblock In \emph{Causal Learning and Reasoning}, pages 423--461. PMLR, 2024.

\bibitem[Jia et~al.(2023)Jia, Oli, Anderson, Duff, Li, and Ravi]{jia2023short}
Su~Jia, Nishant Oli, Ian Anderson, Paul Duff, Andrew~A Li, and Ramamoorthi Ravi.
\newblock Short-lived high-volume bandits.
\newblock In \emph{International Conference on Machine Learning}, pages 14902--14929. PMLR, 2023.

\bibitem[Jia et~al.(2024)Jia, Frazier, and Kallus]{jia2024multi}
Su~Jia, Peter Frazier, and Nathan Kallus.
\newblock Multi-armed bandits with interference.
\newblock \emph{arXiv preprint arXiv:2402.01845}, 2024.

\bibitem[Lagr{\'e}e et~al.(2016)Lagr{\'e}e, Vernade, and Cappe]{lagree2016multiple}
Paul Lagr{\'e}e, Claire Vernade, and Olivier Cappe.
\newblock Multiple-play bandits in the position-based model.
\newblock \emph{Advances in Neural Information Processing Systems}, 29, 2016.

\bibitem[Lattimore et~al.(2016)Lattimore, Lattimore, and Reid]{lattimore2016causal}
Finnian Lattimore, Tor Lattimore, and Mark~D Reid.
\newblock Causal bandits: Learning good interventions via causal inference.
\newblock \emph{Advances in neural information processing systems}, 29, 2016.

\bibitem[Lattimore and Szepesv{\'a}ri(2020)]{lattimore2020bandit}
Tor Lattimore and Csaba Szepesv{\'a}ri.
\newblock \emph{Bandit algorithms}.
\newblock Cambridge University Press, 2020.

\bibitem[Leung(2023)]{leung2023network}
Michael~P Leung.
\newblock Network cluster-robust inference.
\newblock \emph{Econometrica}, 91\penalty0 (2):\penalty0 641--667, 2023.

\bibitem[Li and Wager(2022)]{li2022random}
Shuangning Li and Stefan Wager.
\newblock Random graph asymptotics for treatment effect estimation under network interference.
\newblock \emph{The Annals of Statistics}, 50\penalty0 (4):\penalty0 2334--2358, 2022.

\bibitem[Merlis and Mannor(2020)]{merlis2020tight}
Nadav Merlis and Shie Mannor.
\newblock Tight lower bounds for combinatorial multi-armed bandits.
\newblock In \emph{Conference on Learning Theory}, pages 2830--2857. PMLR, 2020.

\bibitem[Perrault et~al.(2020)Perrault, Boursier, Valko, and Perchet]{perrault2020statistical}
Pierre Perrault, Etienne Boursier, Michal Valko, and Vianney Perchet.
\newblock Statistical efficiency of thompson sampling for combinatorial semi-bandits.
\newblock \emph{Advances in Neural Information Processing Systems}, 33:\penalty0 5429--5440, 2020.

\bibitem[Pouget-Abadie et~al.(2019)Pouget-Abadie, Aydin, Schudy, Brodersen, and Mirrokni]{pouget2019variance}
Jean Pouget-Abadie, Kevin Aydin, Warren Schudy, Kay Brodersen, and Vahab Mirrokni.
\newblock Variance reduction in bipartite experiments through correlation clustering.
\newblock \emph{Advances in Neural Information Processing Systems}, 32, 2019.

\bibitem[Rosenbaum(2007)]{rosenbaum2007interference}
Paul~R Rosenbaum.
\newblock Interference between units in randomized experiments.
\newblock \emph{Journal of the american statistical association}, 102\penalty0 (477):\penalty0 191--200, 2007.

\bibitem[Rubin(1980)]{rubin1980randomization}
Donald~B Rubin.
\newblock Randomization analysis of experimental data: The fisher randomization test comment.
\newblock \emph{Journal of the American statistical association}, 75\penalty0 (371):\penalty0 591--593, 1980.

\bibitem[Shahverdikondori et~al.(2025)Shahverdikondori, Abouei, Rezaeimoghadam, and Kiyavash]{shahverdikondori2025optimal}
Mohammad Shahverdikondori, Amir~Mohammad Abouei, Alireza Rezaeimoghadam, and Negar Kiyavash.
\newblock Optimal best arm identification with post-action context.
\newblock \emph{arXiv preprint arXiv:2502.03061}, 2025.

\bibitem[Shahverdikondori et~al.(2026)Shahverdikondori, Elahi, Thiran, and Kiyavash]{shahverdikondori2026neighborhood}
Mohammad Shahverdikondori, Sepehr Elahi, Patrick Thiran, and Negar Kiyavash.
\newblock Neighborhood-aware graph labeling problem.
\newblock \emph{arXiv preprint arXiv:2602.08098}, 2026.

\bibitem[Tchetgen and VanderWeele(2012)]{tchetgen2012causal}
Eric J~Tchetgen Tchetgen and Tyler~J VanderWeele.
\newblock On causal inference in the presence of interference.
\newblock \emph{Statistical methods in medical research}, 21\penalty0 (1):\penalty0 55--75, 2012.

\bibitem[Tewari and Murphy(2017)]{mab2-tewari2017ads}
Ambuj Tewari and Susan~A Murphy.
\newblock From ads to interventions: Contextual bandits in mobile health.
\newblock \emph{Mobile health: sensors, analytic methods, and applications}, pages 495--517, 2017.

\bibitem[Tzeng et~al.(2023)Tzeng, Wang, Proutiere, and Lu]{tzeng2023closing}
Ruo-Chun Tzeng, Po-An Wang, Alexandre Proutiere, and Chi-Jen Lu.
\newblock Closing the computational-statistical gap in best arm identification for combinatorial semi-bandits.
\newblock \emph{Advances in Neural Information Processing Systems}, 36:\penalty0 18391--18403, 2023.

\bibitem[Ugander and Yin(2023)]{ugander2023randomized}
Johan Ugander and Hao Yin.
\newblock Randomized graph cluster randomization.
\newblock \emph{Journal of Causal Inference}, 11\penalty0 (1):\penalty0 20220014, 2023.

\bibitem[Ugander et~al.(2013)Ugander, Karrer, Backstrom, and Kleinberg]{ugander2013graph}
Johan Ugander, Brian Karrer, Lars Backstrom, and Jon Kleinberg.
\newblock Graph cluster randomization: Network exposure to multiple universes.
\newblock In \emph{Proceedings of the 19th ACM SIGKDD international conference on Knowledge discovery and data mining}, pages 329--337, 2013.

\bibitem[Van~der Vaart and A.~Wellner(1996)]{vanderVaart1996convergence}
Aad~W. Van~der Vaart and Jon A.~Wellner.
\newblock Weak convergence and empirical processes: Introduction to nonparametric estimation.
\newblock \emph{Springer Series in Statistics.}, 1996.

\bibitem[Wang and Chen(2017)]{wang2017improving}
Qinshi Wang and Wei Chen.
\newblock Improving regret bounds for combinatorial semi-bandits with probabilistically triggered arms and its applications.
\newblock \emph{Advances in Neural Information Processing Systems}, 30, 2017.

\bibitem[Xu et~al.(2024)Xu, Lu, and Song]{contextual-xu2024linear}
Yang Xu, Wenbin Lu, and Rui Song.
\newblock Linear contextual bandits with interference.
\newblock \emph{arXiv preprint arXiv:2409.15682}, 2024.

\bibitem[Yang et~al.(2016)Yang, Draper, and Nowak]{yang2016learning}
Jing Yang, Stark~C Draper, and Robert Nowak.
\newblock Learning the interference graph of a wireless network.
\newblock \emph{IEEE Transactions on Signal and Information Processing over Networks}, 3\penalty0 (3):\penalty0 631--646, 2016.

\bibitem[Yu et~al.(2022)Yu, Airoldi, Borgs, and Chayes]{yu2022estimating}
Christina~Lee Yu, Edoardo~M Airoldi, Christian Borgs, and Jennifer~T Chayes.
\newblock Estimating the total treatment effect in randomized experiments with unknown network structure.
\newblock \emph{Proceedings of the National Academy of Sciences}, 119\penalty0 (44):\penalty0 e2208975119, 2022.

\bibitem[Yuan et~al.(2021)Yuan, Altenburger, and Kooti]{yuan2021causal}
Yuan Yuan, Kristen Altenburger, and Farshad Kooti.
\newblock Causal network motifs: Identifying heterogeneous spillover effects in a/b tests.
\newblock In \emph{Proceedings of the Web Conference 2021}, pages 3359--3370, 2021.

\bibitem[Zhang and Wang(2024)]{zhang2024online}
Zhiheng Zhang and Zichen Wang.
\newblock Online experimental design with estimation-regret trade-off under network interference.
\newblock \emph{arXiv preprint arXiv:2412.03727}, 2024.

\end{thebibliography}

\newpage
\appendix
\onecolumn
\section{Omitted Proofs}\label{apx:proofs}

\subsection{Proofs of Section \ref{sec: upper}} \label{sec: proofs upper}

\thupper*
\begin{proof}
    By the definition of $A_{t+1}$, we have: 

 \begin{equation}\label{eq: green}
      \sum_{i \in [N]} \hat{\mu}_{ti}(A_t)+ \sum_{j \in [M]} \sqrt{2\log(2/\delta) \frac{m_j}{\npjt}} \geq \sum_{i \in [N]} \hat{\mu}_{ti}(A^*)+ \sum_{j \in [M]} \sqrt{2\log(2/\delta) \frac{m_j}{\npjs}},
 \end{equation}

where $A^* \in \argmax_{A \in [k]^N} \sum_{i \in [N]} \mu_i(A)$ denotes the optimal treatment.
We define the good event $G$, ensuring that the empirical mean reward is close to the true mean for all times $t$, partitions $j$, and treatments $A$:

$$G \coloneqq \mathbbm{1} \left \{ \forall t \in [T], j \in [M], A \in [k]^N:
\left| \sum_{i \in P_j}\hat{\mu}_{ti}(A) - \mu _{ti}(A)\right| \leq \sqrt{2\log(2/\delta) \frac{m_j}{\npj}} \right \}.$$

Applying the Hoeffding's inequality, we obtain:
\begin{equation*}
    P\left(\left| \sum_{i \in P_j} \hat{\mu}_{i}(A) - \mu_{i}(A)\right| \geq \sqrt{ 2\log(2/\delta) \frac{m_j}{\npj}} \right ) \leq \delta.
\end{equation*}

By a union bound over all $t$, $j$ and $A$:
 \begin{equation*}
     \begin{aligned}
         P (G^c) & \leq \sum_{t \in [T]} \sum_{j \in [M]} \sum_{ A \in [k]^N} P \left( \left| \sum_{i \in P_j} \hat{\mu}_{ti}(A) - \mu_{ti}(A)\right| \geq \sqrt{2\log(2/\delta) {\frac{m_j}{\npj}}}  \right)\\
         & \leq T \sum_{j \in [M]} k^{D_j+1} \delta,
     \end{aligned}
 \end{equation*}
 where $D_j$ denotes the degree of each unit in the $j$-th partition.

By the law of total expectations, we can write the regret under event $G$ as:
\begin{equation*}
 \begin{aligned}
      N \cdot {Reg}_T  &= \sum_{t \in [T]} \sum_{i \in [N]} \mu_{ti}(A^*) -  \mu_{i} (A_t) \\
      & \overset{(a)}{\leq}  \sum_{t \in [T]} \sum_{i \in [N]}   \hat{\mu}_{ti}(A^*)+ \sum_{j \in [M]} \sqrt{2\log(2/\delta) \frac{m_j}{\npjs}} -  \left[ \hat{\mu}_{ti}(A_t)- \sum_{j\in [M]} \sqrt{2\log(2/\delta) \frac{m_j}{\npjt}} \right]\\
      & \overset{(b)}{\leq} 2 \sqrt{2\log(2/\delta)} \sum_{t \in [T]}\sum_{j \in [M]} \sqrt{\frac{m_j}{\npjt}}
 \end{aligned}
 \end{equation*}
 where $(a)$ and $(b)$ hold by definition of event $G$ and \eqref{eq: green}, respectively.

Next, we bound $\sum_{t \in [T]} \frac{1}{\sqrt{\npjt}}$ as

\begin{equation}
    \begin{aligned}
        \sum_{t \in [T]} \frac{1}{\sqrt{\npjt}} 
        &=  \sum_{ A \in [k]^{D_j+1+1}}  \sum_{t \in [T]} \mathbbm 1 \{ \forall i \in P_j:  A_{t N(i)} = A_{N(i)}\} \frac{1}{ \sqrt{\npj}} \\
        &= \sum_{A \in [k]^{D_j+1}} \sum_{t \in [\npjT]} \frac{1}{\sqrt{t}} \\
        &\overset{(a)}{\leq} 2 \sum_{A \in [k]^{D_j+1}} \sqrt{ \npjT} \\
        & \overset{(b)}{\leq} 2 \sqrt{k^{D_j+1}  \sum_{A \in [k]^{D_j+1}}  \npjT}\\
        & \leq 2  \sqrt{Tk^{D_j+1} }
    \end{aligned}
\end{equation}
where $(a)$ holds since $\forall n \in \mathbb{N}: \sum_{i \in [n]} 1/\sqrt{i} \leq 2 \sqrt{n}$, and $(b)$ follows from Jensen inequality.

We prove the inequality  $\sum_{i=1}^{n} \frac{1}{\sqrt{i}} \leq 2\sqrt{n} $ using integral approximation.
Since the function \( f(x) = \frac{1}{\sqrt{x}} \) is decreasing,
$$\sum_{i=1}^{n} \frac{1}{\sqrt{i}} \leq 1 + \int_1^n \frac{1}{\sqrt{x}} \,dx.
$$
Evaluating the integral,
\[
\int_1^n \frac{1}{\sqrt{x}} \,dx = 2\sqrt{n} - 2.
\]
Thus,
\[
\sum_{i=1}^{n} \frac{1}{\sqrt{i}} \leq 1 + (2\sqrt{n} - 2) \leq 2\sqrt{n}.
\]

Therefore, the regret under event $G$ is: 
\begin{equation}
     N \cdot {Reg}_T  \leq 4 \sqrt{2\log(2/\delta)T} \sum_{j \in [M]} \sqrt{m_j k^{D_j+1}} 
\end{equation}

Thus, by choosing $ \delta= \frac{1}{T^2N \sum_{j \in [M]} k^{D_j+1}}$ we get the following by the law of total probability:

\begin{equation*}
\begin{aligned}
    {Reg}_T 
     & \leq 4/{N} \sqrt{2\log(2/\delta)T} \sum_{j \in [M]}  \sqrt{m_j k^{D_j+1}}  + T (T \sum_{j \in [M]}k^{D_j+1} \delta) \\
     & \in  \mathcal{O}\left( \sqrt{ \frac{T}{N^2} \log (T^2N  \sum_{j \in [M]}k^{D_j+1})} \sum_{j \in [M]}  \sqrt{m_j k^{D_j+1}}\right).
\end{aligned}
\end{equation*}
Therefore, since $T \geq k^{\Delta+1}$, it satisfies:
\begin{equation*}
    {Reg}_T  \in  \mathcal{O}\left( \sqrt{ \frac{T}{N^2} \log (TN)} \sum_{j \in [M]}  \sqrt{m_j k^{D_j+1}}\right).
\end{equation*}

\end{proof}

\subsection{Proofs of Section \ref{sec: lower}}\label{sec: proof lower}

\theoremlowerbound*
\begin{proof}
    To prove this lower bound, we employ a change-of-measure argument, a well-known technique in the multi-armed bandit literature for deriving lower bounds. Consider an instance $\V$ with interference graph $\G$, Gaussian reward noises with unit variance, and the following reward means for each unit $i \in [N]$:
    $$
    \mui = 
    \begin{cases}
    \Delta_i & \text{if }  A_{N(i)} = (1, 1, \ldots, 1), \\
    0 & \text{otherwise,}
    \end{cases}
    $$
    
    where $\Delta_i$ is a positive real number to be determined later.
    In this case, the optimal treatment is to assign treatment $1$ to all units, yielding an expected reward of $\sum_{i \in [N]} \Delta_i$.
    
    Now, fix a policy $\pi$ operating on instance $\V$. 
    For each treatment $A \in [k]^N$, let $T_{\pi}(A) = \left(t_{\pi,1}(A), t_{\pi,2}(A), \ldots, t_{\pi,N}(A) \right)$, where $t_{\pi,i}(A)$ is the expected number of times policy $\pi$ applies a treatment assignment such that unit $i$ and its neighbors receive treatment $A_{N(i)}$ over $T$ rounds.
    For simplicity, we denote $T_{\pi}(A)$ and $t_{\pi,i}(A)$ as $T(A)$ and $t_i(A)$, respectively.
    
    Let $S \subset [k]^N$ be the set of all treatments where no unit receives treatment $1$, therefore, $|S| = (k-1)^N$.
    The following lemma implies that there is always an under-explored treatment.

    \begin{lemma}\label{lem: under-explored}
        For any policy $\pi$ and any set of values $\Delta_i$ for $i \in [N]$, there exists a treatment $A' \in S$ such that:
        \begin{align*}
            \sum_{i \in [N]} t_i(A') \Delta_i^2 \leq T \left( \sum_{i \in [N]} \frac{\Delta_i^2}{(k-1)^{d_i+1}} \right).
        \end{align*}
    \end{lemma}
    \begin{proof}
    To prove this lemma, we employ a double-counting technique. Consider a matrix with $(k-1)^N$ rows and $N$ columns where row $i$ corresponds to a treatment $A_i \in S$ and column $i$ corresponds to unit $i$. On the element in $j$-th row and $i$-th column of the matrix, we write $t_i(A_j) \Delta_i^2$ as follows: 
    \begin{align*}
        \mathbf{M} \coloneqq \begin{bmatrix}
        t_1(A_1)\Delta_1^2 & t_2(A_1)\Delta_2^2 & \cdots & t_N(A_1)\Delta_N^2 \\
        t_1(A_2)\Delta_1^2 & t_2(A_2)\Delta_2^2 & \cdots & t_N(A_2)\Delta_N^2 \\
        \vdots & \vdots & \ddots & \vdots \\
        t_1(A_{k-1^N})\Delta_1^2 & t_2(A_{k-1^N})\Delta_2^2 & \cdots & t_N(A_{k-1^N})\Delta_N^2
    \end{bmatrix}.
    \end{align*}
    
    To calculate the sum of the numbers in the first column, we have
    \begin{align*}
        \sum_{A_i \in S} t_1(A_i) \Delta_1^2 = \Delta_1^2 (k-1)^{N - (d_i+1)} \sum_{A_{N(1)} \in S_1} t( A_{N(1)}),
    \end{align*}
    where $S_1 = \{2,3,\dots,k\}^{d_i+1}$ is the set of all treatments that can be assigned to unit one and its neighbors without using treatment $1$ and $t(A_{N(1)})$ is the expected number of times that unit one and its neighbors are assigned $A_{N(1)}$ during $T$ rounds of interaction between $\pi$ and $\V$. The equation is true because of the symmetry in the problem which implies that each combination of the treatment of unit $1$ and its neighbors exists in $(k-1)^{N - (d_i+1)}$ number of members of $S$. On the other hand, we have
    \begin{align*}
        \sum_{A_{N(1)} \in S_1} t(A_{N(1)}) = T.
    \end{align*}
    
    Writing the same equation for all the columns implies that the sum of the numbers in the whole matrix is equal to
    \begin{align*}
        T \sum_{i \in [N]} \Delta_i^2 (k-1)^{N - (d_i+1)}.
    \end{align*}
    
    Dividing this number by the number of rows shows that there exists a row $j$ such that
    \begin{align*}
        \sum_{i \in [N]} t_i(A_j)\Delta_i^2 \leq T \left( \sum_{i \in [N]} \frac{\Delta_i^2}{(k-1)^{d_i+1}} \right). 
    \end{align*}
\end{proof}

    Now, based on the under-explored treatment $A'$, we design a \textit{confusing} instance $\V'$ which for each treatment $A$ has the mean expected rewards as follows: 
    \begin{align*}
        \mu'(A_{N(i)}) = \begin{cases}
            \Delta_i \quad &\text{ if } A_{N(i)} = (1,1,\ldots, 1), \\ 
            2\Delta_i \quad &\text{ if } A_{N(i)} = A'_{N(i)}, \\
            0 \quad &\text{ otherwise}. 
        \end{cases}
    \end{align*}
    We denote $\pr_{\V}$ and $\pr_{\V'}$ the probability measures over the bandit model induced by $T$ rounds of interaction between the policy $\pi$ and the instances $\V$ and $\V'$, respectively.
    In this case, using the divergence decomposition lemma \citep[see Lemma $15.1$ in][]{lattimore2020bandit} and the fact that the KL-divergence (denoted by $D_{KL}(\cdot \| \cdot)$ ) between two $1$-Gaussian distributions with means $\mu_1$ and $\mu_2$ is equal to $(\mu_1 - \mu_2)^2 / 2$, we derive the following equality
    \begin{align}
        D_{KL}(\pr_{\V} \| \pr_{\V'}) = \sum_{i \in [N]} t_i(A') \frac{(2\Delta_i)^2}{2} = 2 \sum_{i \in [N]} t_i(A') \Delta_i^2.
    \end{align}
    Using Lemma \ref{lem: under-explored}, we have
    \begin{align} \label{eq: kl-bound}
        D_{KL}(\pr_{\V}\| \pr_{\V'}) \leq 2T \left( \sum_{i \in [N]} \frac{\Delta_i^2}{(k-1)^{d_i+1}} \right). 
    \end{align}

    For each $t \in [T]$, define the event $\ev_t$ as:  
    $$
    \ev_t \coloneqq \left\{ \mathbb{E}\left[\sum_{i \in \ocal_t} \Delta_i \right] \geq \frac{1}{2} \sum_{i \in [N]} \Delta_i \right\},
    $$
    where $\ocal_t \coloneqq \{ i \mid A_{tN(i)} = (1, \ldots, 1) \}$  denotes the set of units $i$ that receive treatment $1$ at time $t$, along with all their neighbors.  
    This event signifies that in round $t$, the expected sum of $\Delta_i$ for such units $i$ is at least half of the total $\sum_{i \in [N]} \Delta_i$.
    We further define the event $\ev$ as:
    \begin{align*}
        \ev \coloneqq \left\{ \sum_{t \in [T]} \mathbbm{1}_{\ev_t} \geq \frac{T}{2} \right\},
    \end{align*}
    indicating that $\ev_t$ occurs in at least half of the $T$ rounds.
    Using the event $\ev$, we can bound the regret of policy $\pi$ for both instances $\V$ and $\V'$. 
    For $\V$ we have:
    \begin{align*}
        {Reg}_T(\pi, \V) &\geq \pr_{\V}(\ev^c)  \left({Reg}_T(\pi, \V) \mid \ev^c \right) \\
        & \geq \pr_{\V}(\ev^c) \frac{T}{4N} \sum_{i \in [N]} \Delta_i,
    \end{align*}
    where, with a slight abuse of notation, we use ${Reg}_T(\pi, \V) \mid \ev^c$ to denote regret (which is an expectation) conditioned on the event $\ev^c$. The second line holds because, on $\ev^c$, the algorithm incurs a regret of at least $\sum_{i \in [N]} \Delta_i$ in at least half of the rounds. Consequently, the regret of $\pi$ interacting with $\V$ is at least $\frac{T}{4N} \sum_{i \in [N]} \Delta_i$.

    Similarly, for $\V'$, the regret of $\pi$ can be bounded using $\ev$ as:
    \begin{align} \label{eq: lower-reg-V}\nonumber
        {Reg}_T(\pi, \V') &\geq \pr_{\V'}(\ev)  \left({Reg}_T(\pi, \V')  \mid \ev \right) \\ 
        & \geq \pr_{\V'}(\ev) \frac{T}{4N} \sum_{i \in [N]} \Delta_i.
    \end{align}
    On the other hand, Bretagnolle-Huber inequality \citep{vanderVaart1996convergence} implies:
    \begin{align} \label{eq: lower-reg-V'} \nonumber
        \pr_{\V}(\ev^c) + \pr_{\V'}(\ev) &\geq \frac{1}{2} \exp(-D_{KL}(\pr_{\V} \| \pr_{\V'})) \\ 
        &\geq \frac12 \exp{\left( - 2 T \sum_{i \in [N]} \frac{\Delta_i^2}{(k-1)^{d_i+1}} \right)},
    \end{align}
    where the second line holds by Equation \eqref{eq: kl-bound}.
    From Equations \eqref{eq: lower-reg-V} and \eqref{eq: lower-reg-V'}, we have:
    \begin{align*}
        {Reg}_T(\pi, \V) + {Reg}_T(\pi, \V') \geq \frac{T}{8N} (\sum_{i \in [N]} \Delta_i) \exp{\left( - 2T \sum_{i \in [N]} \frac{\Delta_i^2}{(k-1)^{d_i+1}} \right)}.
    \end{align*}

    Next, for each $j \in [M]$ and each $l \in P_j$, set the value of $\Delta_l$ to:
    \begin{align*}
        \Delta_l = \sqrt{\frac{(k-1)^{D_j+1}}{TMm_j}},
    \end{align*}
    which satisfies 
    \begin{align*}
        T \sum_{i \in [N]} \frac{\Delta_i^2}{(k-1)^{d_i+1}} = T \sum_{j \in [M]} \sum_{l \in P_j} \frac{1}{Mm_j} = T \sum_{j \in [M]} \frac{1}{M} = 1,
    \end{align*}
    and
    \begin{align*}
        \sum_{i \in [N]} \Delta_i &= \sqrt{\frac{1}{TM}} \sum_{j \in [M]}  \sum_{l \in P_j} {\sqrt{\frac{(k-1)^{D_j+1}}{m_j}}} \\ 
        & = \sqrt{\frac{1}{TM}} \sum_{j \in [M]} \sqrt{m_j (k - 1)^{D_j + 1}}. 
    \end{align*}

    Note that by the assumption $T \geq \frac{4 (k-1)^{\Delta + 1}}{M}$, we have $\Delta_i \leq \frac12$ which implies that the mean rewards are in $[0,1]$.
    
    Now, note that if $k > \frac{2^{\frac{1}{\Delta+1}}}{2^{\frac{1}{\Delta+1}} - 1}$, then $(k-1)^{D_j+1} > \frac12k^{D_j+1}$.
    Using this, there exists a universal constant $C$ such that at least one of ${Reg}_T(\pi, \V)$ and $ {Reg}_T(\pi, \V')$ is greater than
    \begin{align*}
        C \left( \sqrt{\frac{T}{N^2M}} \sum_{j \in [M]} \sqrt{m_j k^{D_j+1}} \right). 
    \end{align*}
    
Therefore, 

   \begin{align*}
        {Reg}_T(\pi) \in \Omega \left( \sqrt{\frac{T}{N^2M}} \sum_{j \in [M]} \sqrt{m_jk^{D_j+1}} \right).
    \end{align*}
    
\end{proof}

\theoremsparselowerbound*
\begin{proof}
To prove this theorem, we first provide two important lemmas.
\begin{restatable}{lemma}{edgeremoval} \label{lem: edge-removal}
        For a graph $\G$, let $\G'$ be the graph obtained after removing an arbitrary edge from $\G$.
        Then, for any policy $\pi$, the worst-case regret of interacting with bandit instances on $\G$ is at least that of interacting with bandit instances on $\G'$.
        That is,
        $$\sup_{\V \sim \G} {Reg}_T(\pi, \V) \geq \sup_{\V' \sim \G'} {Reg}_T(\pi, \V'),$$
        where $V \sim \G$ denotes a bandit instance whose interference graph is $\G$.
    \end{restatable}

\begin{proof}
    Assume $\G'$ is obtained by removing an edge between nodes $i$ and $j$. To prove this lemma, consider any instance $\V'$ with the interference graph $\G'$. 
    Construct an instance $\V$ with the interference graph $\G$ such that the reward distribution for each treatment $A_{N(i)}$ for unit $i$ is identical to its distribution in $\V'$.

    In this construction, the reward distribution for the unit $i$ becomes independent of the node $j$, and the edge $i - j$ has no impact on the rewards. Consequently, no algorithm can achieve a lower regret on $\V$ compared to $\V'$. This implies that for every instance with the interference graph $\G'$, a harder instance exists with the interference graph $\G$, which completes the proof.
\end{proof}


    \begin{restatable}{lemma}{onenodelowerbound} \label{lem: one-node-lower-bound}
        Let $\G$ be a graph with $N$ nodes where $d_i$ is the degree of node $i$.
        If $T \geq k^{d_i+1} - 1$, for any policy $\pi$ and each $i \in [N]$, there exists a bandit instance with interference graph $\G$, $1$-Gaussian reward distributions and means in $[0,1]$ such that:
        \begin{align*}
            {Reg}_T(\pi) \in \Omega \left( \sqrt{\frac{T k^{d_i+1}}{N^2}} \right).
        \end{align*}
    \end{restatable}
\begin{proof}
    To prove this lemma, observe that the reward distribution of unit $i$ for each treatment $A_{N(i)}$ assigned to unit $i$ and its neighbors is independent of other treatments. This allows us to construct a corresponding classic multi-armed bandit instance with $k^{d_i+1}$ arms, where each arm represents a treatment $A_{N(i)}$ and follows the same reward distribution.

    In this scenario, based on the classic lower bound in the multi-armed bandit literature \citep{lattimore2020bandit}, for any policy interacting with a bandit with $K$ arms, there exists an instance where the regret ${Reg}_T(\pi)$ is at least $\sqrt{TK}$. Applying this to our problem, where total regret is defined as the average regret of units, implies that for any policy $\pi$, there exists an instance with the interference graph $\G$ such that:
    $$
    {Reg}_T(\pi) = \Omega \left( \sqrt{\frac{T k^{d_i+1}}{N^2}} \right).
    $$
\end{proof}

We prove the theorem using the aforementioned lemmas as follows.
    For any set $S = \{s_1, s_2, \ldots, s_m\} \in DI(\G)$, let $\G_S$ be the graph obtained from $\G$ by removing all edges between nodes that are outside of $S$ (i.e., edges where neither endpoint belongs to $S$).
    By Lemma \ref{lem: edge-removal}, the worst-case regret for instances with interference graph $\G$ is at least as large as the worst-case regret for instances with $\G_S$. 
    The graph $\G_S$ consists of edges where at least one endpoint lies in $S$.
    Since $S \in DI(\G)$, the resulting graph $\G_S$ consists of $m$ disjoint connected components $\G_{s_1}, \G_{s_2} \ldots, \G_{s_m}$.
    Each component $\G_{s_i}$ is a star graph with $d_{s_i} + 1$ nodes.
    
    By Lemma \ref{lem: one-node-lower-bound}, for any policy $\pi$ interacting with $\G_{s_i}$, there exists a bandit instance where:
    \begin{align*}
        {Reg}_T(\pi) \in \Omega \left( \sqrt{\frac{T k^{d_{s_i}+1}}{N^2}} \right).
    \end{align*}
    Since the components $\G_{s_i}$ are disjoint, the total regret over $\G_S$ equals the sum of the regrets for each component.
    Therefore, for any policy $\pi$ and any $S \in DI(\G)$, there exists an instance such that:
    \begin{align*}
        {Reg}_T(\pi) \in \Omega \left( \sqrt{\frac{T}{N^2}} \sum_{i \in S} \sqrt{k^{d_i + 1}} \right).
    \end{align*}
 This concludes the proof.
 \end{proof}

\cliquesparse*
\begin{proof}
    For each cluster $C_i$, let $|C_i| = c_i$. It is given that the number of edges between the nodes in $C_i$ and $C_j$ is at most $r$ for each pair $i, j$. This implies that, within cluster $C_i$, there are at most $r(R-1)$ nodes with neighbors outside of $C_i$. Let $D_i$ denote the set of nodes in $C_i$ that have no neighbors outside $C_i$. Thus, we have:

    $$
    |D_i| \geq c_i - r(R-1).
    $$
    
    Since $C_i$ is a complete graph, it holds that:  
    $$
    \forall l \in D_i: N(l) = C_i.
    $$
    
    This indicates that for all nodes $l \in D_i$, the set $N(l)$ is identical, which implies that all such nodes belong to the same partition. Furthermore, there are at most $r(R-1)$ nodes in $C_i \setminus D_i$, meaning that in total these nodes can form:  
    $$
    \sum_{i \in [R]} |C_i \setminus D_i| \leq \sum_{i \in [R]} r(R-1) = rR(R-1)
    $$  
    distinct partitions. This implies that the total number of partitions $M(\G)$ is bounded as:  
    $$
    M(\G) \leq R + rR(R-1).
    $$

\end{proof}

\constantchrom*
\begin{proof}
    To prove the lemma, we provide a coloring method using at most $\Delta^2 + 1$ colors, ensuring that two nodes with the same color are neither adjacent nor share a common neighbor. Consider an arbitrary order $\{v_1, v_2, \ldots, v_N\}$ on the nodes, and let the available colors be $\{c_1, c_2, \ldots, c_{\Delta^2+1}\}$. Start with $v_1$, and for each $v_i$, assign the smallest color $c_j$ such that no node already colored with $c_j$ is adjacent to $v_i$ or shares a common neighbor with $v_i$. 
    
    To prove that this coloring method is valid and does not require more than $\Delta^2 + 1$ colors, assume the process stops at $v_i$ because no valid color is available. This would mean $v_i$ has more than $\Delta^2$ nodes that are either adjacent to it or share a common neighbor. 
    However, since the maximum degree is $\Delta$, $v_i$ has at most $\Delta$ neighbors, and each neighbor can have at most $\Delta - 1$ other neighbors. This totals at most $\Delta^2$ nodes, contradicting the assumption that more than $\Delta^2$ nodes are involved. Thus, the coloring method works as intended.
    This shows that for constant values of $\Delta$ (independent of $N$ and $k$), the value $\Delta^2 + 1$ is also constant.
    
\end{proof}

\subsection{Proofs of Section \ref{sec: unkown}} \label{sec: proofs uknown}

\unknownLB*
\begin{proof}
    Fix an arbitrary graph $\G$ and consider the instance $\V$ with interference graph $\G$, 1-Gaussian rewards, and the following reward means for each unit $i$ under treatment $A$,
    $$
    \mui = 
    \begin{cases}
    \Delta & \text{if }  A_{N(i)} = (1, 1, \ldots, 1), \\
    0 & \text{otherwise,}
    \end{cases}
    $$
    where $\Delta$ is a positive real number to be determined later. The optimal treatment is to assign treatment $1$ to everyone, which yields an expected reward of $N\Delta$.  
    
    Let $S \subset [k]^N$ denote the set of all treatments where no unit receives treatment $1$. For each treatment $A' \in S$, we construct the instance $\V_{A'}$ with the following properties: the interference graph is $K_N$ (complete graph with $N$ nodes), $1$-Gaussian rewards with means
    \begin{align*}
        \mu'(A_{N(i)}) = \begin{cases}
            \Delta \quad &\text{ if } A_{N_{\G}(i)} = (1,1,\ldots, 1), \\ 
            2\Delta \quad &\text{ if } A = A', \\
            0 \quad &\text{ otherwise}, 
        \end{cases}
    \end{align*}
    where $A_{N_{\G}(i)}$ denotes the treatment assigned to the neighborhood of node $i$ in graph $\G$. The instance $\V_{A'}$ differs from the base instance $\V$ only in the units' rewards assigned to treatment $A'$, which is now the optimal treatment in $\V_{A'}$
    
    Now consider a legit policy $\pi$. During the interaction of $\pi$ with a problem instance under an unknown interference graph, we define the event $\mathcal{E}$ as:

    \begin{align*}
        \mathcal{E} = \sum_{t \in [T]} n_t \geq \frac{TN}{2},
    \end{align*}
    where $n_t$ denotes the expected number of nodes $i$ in round $t$ whose entire neighborhood receives treatment $1$. Note that under $\mathcal{E}$, the regret of $\pi$ when interacting with $\V_{A'}$ is at least $\frac{T\Delta}{2}$, and under $\mathcal{E}^c$, the regret of $\pi$ interacting with $\V$ is at least $\frac{T\Delta}{2}$. This shows that
    \begin{align*}
        \rtgpi{\G} + \rtgpi{K_N} \geq \rtgpi{\V} + \rtgpi{\V_{A'}} \geq \frac{T \Delta}{2} \left( \pr_{\V}(\mathcal{E}^c) + \pr_{\V_{A'}}(\mathcal{E}) \right), 
    \end{align*}
    where $\pr_{\V}, \pr_{\V_{A'}}$ denote the probability measures over the bandit model induced by $T$ rounds of interaction of the policy $\pi$ and instances $\V$ and $\V'$, respectively, when the interference graph is unknown.
    
    By Bretagnolle-Huber inequality \citep{vanderVaart1996convergence}, we have
    \begin{align*}
        \pr_{\V}(\mathcal{E}^c) + \pr_{\V_{A'}}(\mathcal{E}) \geq \frac12 \exp{\left(-D_{KL} (\pr_{\V} \| \pr_{\V_{A'}}) \right)}.
    \end{align*}

    Since $\V$ and $\V_{A'}$ only differ in the mean rewards of treatment $A'$, the KL divergence between their distributions can be computed as follows using the KL formula for two $1$-Gaussian distributions:
    \begin{align*}
        D_{KL} (\pr_{\V} \| \pr_{\V_{A'}}) = N \frac{(2 \Delta)^2}{2} \E_{\V}[T_{A'}],
    \end{align*}
    where $\E_{\V}[T_{A'}]$ shows the expected number of times $\pi$ plays $A'$ on instance $\V$. Then we have: 
    \begin{align*}
        2 \rtgpi{K_N} \geq \rtgpi{\G} + \rtgpi{K_N} \geq \frac{T\Delta}{4} \exp\left( - 2 N \Delta^2 \E_{\V}[T_{A'}] \right),
    \end{align*}
    where the first inequality holds because $\pi$ is legit. Then, by simple algebra, we have

    \begin{align*}
        \E_{\V}[T_{A'}] \geq \frac{\log\left( \frac{T\Delta}{4} \right) - \log(2 \rtgpi{K_N})}{2 N \Delta^2}. 
    \end{align*}

    By setting $\Delta = \frac{12 \rtgpi{K_N}}{T}$,

    \begin{align*}
        \E_{\V}[T_{A'}] \geq \frac{\log(\frac32) T^2}{2 \times 144 \times N \rtgpi{K_N}^2} = \frac{\log(\frac32)}{288} \frac{T^2}{N \rtgpi{K_N}^2}. 
    \end{align*}

    As this inequality holds for every treatment $A' \in S$, then
    \begin{align*}
        \rtgpi{\G} \geq \rtgpi{\V} &\geq \frac{1}{N} \sum_{A' \in S} \E_{\V}[T_{A'}] (N \Delta) \geq \frac{1}{N} \sum_{A' \in S}\frac{\log(\frac32)}{288} \frac{T^2}{N \rtgpi{K_N}^2} \frac{12 N \rtgpi{K_N}}{T} \\
        & = \frac{|S|}{N}\frac{\log(\frac32)}{288} \frac{12T}{\rtgpi{K_N}}  \\
        & \geq \frac{k^N}{2N}\frac{\log(\frac32)}{288} \frac{12T}{\rtgpi{K_N}} \in \Omega \left( \frac{Tk^N}{N \rtgpi{K_N}} \right) \\
        \Longrightarrow &\rtgpi{\G} \rtgpi{K_N} \in \Omega \left( \frac{T}{N} k^N \right),
    \end{align*}
    where the last inequality holds because $|S| = (k-1)^N > \frac{k^N}{2}$ for $k > \frac{2^{\frac{1}{N}}}{2^{\frac{1}{N}} - 1}$.
\end{proof}


\end{document}